\documentclass{article}

     \PassOptionsToPackage{numbers, compress}{natbib}

    \usepackage[preprint]{neurips_2020}

\usepackage[utf8]{inputenc} 
\usepackage[T1]{fontenc}    
\usepackage{hyperref}       
\usepackage{url}            
\usepackage{booktabs}       
\usepackage{amsfonts}       
\usepackage{nicefrac}       
\usepackage{microtype}      

\usepackage{amsmath,amsthm}
\usepackage{comment}
\usepackage{subfloat}
\usepackage{subfig}
\usepackage{algorithm}
\usepackage{algorithmic}

\newtheorem{theorem}{Theorem}
\newtheorem{lemma}{Lemma}
\newtheorem{definition}{Definition}

\newtheorem{assumption}{Assumption}
\newtheorem{remark}{Remark}

\newcommand*\cM{\mathcal{M}}
\newcommand*\cU{\mathcal{U}}
\newcommand*\cB{\mathcal{B}}
\newcommand*\cT{\mathcal{T}}
\newcommand*\cQ{\mathcal{Q}}

\newcommand*\bH{\mathbf{H}}

\newcommand*\bU{\mathbf{U}}
\newcommand*\bV{\mathbf{V}}
\newcommand*\bA{\mathbf{A}}
\newcommand*\bZ{\mathbf{Z}}
\newcommand*\bS{\mathbf{S}}
\newcommand*\bB{\mathbf{B}}

\newcommand*\ma{\mathbf{a}}
\newcommand*\mb{\mathbf{b}}
\newcommand*\w{\mathbf{w}}
\newcommand*\g{\mathbf{g}}
\newcommand*\p{\mathbf{p}}
\newcommand*\x{\mathbf{x}}
\newcommand*\bu{\mathbf{u}}

\newcommand*\z{\mathbf{z}}

\usepackage{comment}
\usepackage{epsf}
\usepackage{graphicx}

\usepackage{color}

\usepackage{mathtools}
\usepackage{amsfonts}
\usepackage{amsthm}
\usepackage{amsmath, bm}
\usepackage{amssymb}

\usepackage{pgfplots}
\pgfplotsset{width=5.3cm,compat=1.9}

\usepackage{textcomp}
\usepackage{xcolor}

\DeclareSymbolFont{extraup}{U}{zavm}{m}{n}
\DeclareMathSymbol{\vardiamond}{\mathalpha}{extraup}{87}

\long\def\comment#1{}

\newcommand{\real}{\ensuremath{\mathbb{R}}}

\newcommand{\R}{\ensuremath{\mathbb{R}}}

\title{Distributed Newton Can Communicate Less and Resist Byzantine Workers}

\author{%
  Avishek Ghosh \\
  Department of EECS,
  UC Berkeley\\
  Berkeley, CA 94720 \\
  \texttt{avishek\_ghosh@berkeley.edu} \\
   \And
   Raj Kumar Maity\\
   College of Information and Computer Sciences\\ 
   UMass Amherst, MA-01002\\
   \texttt{rajkmaity@cs.umass.edu}\\
   \AND
   Arya Mazumdar\\
   College of Information and Computer Sciences\\ 
   UMass Amherst, MA-01002\\
   \texttt{arya@cs.umass.edu}\\
}

\begin{document}

\maketitle

\begin{abstract}
  We develop a distributed second order optimization algorithm that is communication-efficient as well as robust against Byzantine failures of the worker machines. We propose \textsf{COMRADE} (COMunication-efficient and Robust Approximate Distributed nEwton), an iterative second order algorithm, where the worker machines communicate \emph{only once} per iteration with the center machine. This is in sharp contrast with the state-of-the-art distributed second order algorithms like GIANT \cite{giant} and DINGO\cite{dingo}, where the worker machines send (functions of) local gradient and Hessian sequentially; thus ending up communicating twice with the center machine per iteration. Moreover, we show that the worker machines can further compress the local information before sending it to the center. In addition, we employ a simple norm based thresholding rule to filter-out the Byzantine worker machines. We establish the linear-quadratic rate of convergence of \textsf{COMRADE} and establish that the communication savings and Byzantine resilience result in only a small statistical error rate for arbitrary convex loss functions. To the best of our knowledge, this is the first work that addresses the issue of Byzantine resilience in second order distributed optimization. Furthermore, we validate our theoretical results with extensive experiments on synthetic and benchmark LIBSVM \cite{libsvm} data-sets and demonstrate convergence guarantees.
\end{abstract}
\vspace{-4mm}
\section{Introduction}
\vspace{-2.5mm}
In modern data-intensive applications like image recognition, conversational AI and recommendation systems, the size of training datasets has grown in such proportions that distributed computing have become an integral part of machine learning. To this end, a fairly common distributed learning framework, namely \emph{data parallelism}, distributes the (huge) data-sets over multiple \emph{worker machines} to exploit the power of parallel computing. In many applications, such as Federated Learning \cite{federated}, data is stored in users' personal devices and judicious exploitation of the on-device machine intelligence can speed up computation. Usually, in a distributed learning framework, computation (such as processing, training) happens in the worker machines and the local results are communicated to a \emph{center machine} (ex., a parameter server). The center machine updates the model parameters by properly aggregating the local results.

Such distributed frameworks face the following two
fundamental challenges: First, the parallelism gains are often bottle-necked by the heavy communication overheads between worker and the center machines. This issue is further exacerbated where large clusters of worker machines are used for modern deep learning applications using models with millions of parameters (NLP models, such as BERT~\cite{devlin2018bert}, may have well over 100 million parameters). Furthermore, in Federated Learning, this uplink cost is tied to the user's upload bandwidth. Second, the worker machines might be susceptible to errors owing to data crashes, software or hardware bugs, stalled computation or even malicious and co-ordinated attacks. This inherent unpredictable (and potentially adversarial) nature of worker machines is typically modeled as Byzantine failures. As shown in \cite{lamport}, Byzantine behavior a single worker machine can be fatal to the learning algorithm.

Both these challenges, communication efficiency and Byzantine-robustness, have been addressed in a significant number of recent works, albeit mostly separately. For communication efficiency, several recent works \cite{atomo,dme,signsgd,vqsgd,qsgd,terngrad,errorfeed} use quantization or sparsification schemes to compress the message sent by the worker machines to the center machine. An  alternative, and perhaps more natural way  to reduce the  communication  cost (via reducing the  number of iterations) is to use second order optimization algorithms;  which  are  known  to  converge  much  faster  than their first order counterparts. Indeed, a handful of algorithms has been developed using this philosophy, such as DANE \cite{dane}, DISCO \cite{disco}, GIANT \cite{giant} , DINGO \cite{dingo}, Newton-MR \cite{newtonmr}, INEXACT DANE and AIDE \cite{aide}. In a recent work \cite{sparsedingo}, second order optimization and compression schemes are used simultaneously for communication efficiency. 
On the other hand, the problem of developing Byzantine-robust distributed algorithms has also been considered recently (see  \cite{su,feng,chen,dong,dong1,ghosh2019robust,blanchard2017byzantine} ). However, \emph{all} of these papers analyze different variations of the gradient descent, the standard first order optimization algorithm.

In this work, we propose \textsf{COMRADE}, a distributed approximate Newton-type algorithm that communicates less and is resilient to Byzantine workers. Specifically, we consider a distributed setup with $m$ worker machines and one center machine. The goal is to minimize a regularized convex loss $f: \R^d\to \R$, which is additive over the available data points. Furthermore, we assume that $\alpha$ fraction of the worker machines are Byzantine, where $\alpha \in [0,1/2)$. We assume that Byzantine workers can send any arbitrary values to the center machine. In addition, they may completely know the learning algorithm and are allowed to collude with each other. To the best of our knowledge, this is the first paper that addresses the problem of Byzantine resilience in second order optimization.

In our proposed algorithm, the worker machines communicate \emph{only once} per iteration with the center machine. This is in sharp contrast with the state-of-the-art distributed second order algorithms (like GIANT \cite{giant}, DINGO \cite{dingo}, Determinantal Averaging \cite{deter}), which sequentially estimates functions of local gradients and Hessians and communicate them with the center machine. In this way, they end up communicating twice per iteration with the center machine. We show that this sequential estimation is redundant. Instead, in \textsf{COMRADE}, the worker machines only send a $d$ dimensional vector, the product of the inverse of local Hessian and the local gradient. Via sketching arguments, we show that the empirical mean of the product of local Hessian inverse and local gradient is close to the global Hessian inverse and gradient product, and thus just sending the above-mentioned product is sufficient to ensure convergence. Hence, in this way, we save $\mathcal{O}(d)$ bits of communication per iteration. Furthermore, in Section~\ref{sec:compress}, we argue that, in order to cut down further communication, the worker machines can even compress the local Hessian inverse and gradient product. Specifically, we use a (generic) $\rho$-approximate compressor (\cite{errorfeed}) for this, that encompasses sign-based compressors like QSGD \cite{qsgd} and top$_k$ sparsification \cite{stich2018sparsified}.

For Byzantine resilience, \textsf{COMRADE} employs a simple thresholding policy on the norms of the local Hessian inverse and local gradient product. Note that norm-based thresholding is computationally much simpler in comparison to existing co-ordinate wise median or trimmed mean (\cite{ dong}) algorithms. Since the norm of the Hessian-inverse and gradient product determines the \emph{amount} of movement for Newton-type algorithms, this norm corresponds to a natural metric for identifying and filtering out Byzantine workers. 

\vspace{-3mm}
\paragraph{Our Contributions:}
We propose a communication efficient Newton-type algorithm that is robust to Byzantine worker machines. Our proposed algorithm, \textsf{COMRADE} takes as input the local Hessian inverse and gradient product (or a compressed version of it) from the worker machines, and performs a simple thresholding operation on the norm of the said vector to discard $\beta > \alpha$ fraction of workers having largest norm values. We prove the linear-quadratic rate of convergence of our proposed algorithm for strongly convex loss functions. In particular, suppose there are $m$ worker machines, each containing $s$ data points; and let $\mathbf{\Delta}_t = \w_t - \w^*$, where $\w_t$ is the $t$-th iterate of \textsf{COMRADE}, and $\w^*$ is the optimal model we want to estimate. In Theorem 2, we show that
\begin{align*}
\| \mathbf{\Delta}_{t+1}\| \leq \max \{ \Psi^{(1)}_t \| \mathbf{\Delta}_{t}\|, \Psi^{(2)}_t \| \mathbf{\Delta}_{t}\|^2  \}+ (\Psi^{(3)}_t +\alpha) \sqrt{\frac{1}{s}},
\end{align*}
where $\{\Psi^{(i)}_t\}_{i=1}^3$ are quantities dependent on several problem parameters. Notice that the above implies a quadratic rate of convergence when $ \| \mathbf{\Delta}_{t}\| \geq \Psi^{(1)}_t/\Psi^{(2)}_t$. Subsequently, when $\| \mathbf{\Delta}_{t}\|$ becomes sufficiently small, the above condition is violated and the convergence  slows down to a linear rate. The error-floor, which is $\mathcal{O}(1/\sqrt{s})$ comes from the Byzantine resilience subroutine in conjunction with the simultaneous estimation of Hessian and gradient. Furthermore, in Section~\ref{sec:compress}, we consider worker machines compressing the local Hessian inverse and gradient product via a $\rho$-approximate compressor \cite{errorfeed}, and show that the (order-wise) rate of convergence remain unchanged, and the compression factor, $\rho$ affects the constants only.

We experimentally validate our proposed algorithm, \textsf{COMRADE}, with several benchmark data-sets. We consider several types of Byzantine attacks and observe that \textsf{COMRADE} is robust against Byzantine worker machines, yielding better classification accuracy compared to the existing state-of-the-art second order algorithms.

A major technical challenge of this paper is to approximate local gradient and Hessian simultaneously in the presence of Byzantine workers. We use sketching, similar to \cite{giant}, along with the norm based Byzantine resilience technique. Using \emph{incoherence} (defined shortly) of the local Hessian along with concentration results originating from uniform sampling, we obtain the simultaneous gradient and Hessian approximation. Furthermore, ensuring at least one non-Byzantine machine gets trimmed at every iteration of \textsf{COMRADE}, we control the influence of Byzantine workers.

\vspace{-2mm}
\paragraph{Related Work:} \textit{Second order Optimization:} Second order optimization has received a lot of attention in the recent years in the distributed setting owing to its attractive convergence speed. The fundamentals of second order optimization is laid out in \cite{dane}, and an extension with better convergence rates is presented in \cite{aide}. Recently, in GIANT \cite{giant} algorithm, each worker machine computes an approximate Newton direction in each iteration and the center machine averages them to obtain a \emph{globally improved} approximate Newton direction.  Furthermore, DINGO \cite{dingo} generalizes second order optimization beyond convex functions by extending the Newton-MR \cite{newtonmr} algorithm in a distributed setting. Very recently, \cite{deter} proposes Determinantal averaging to correct the inversion bias of the second order optimization. A slightly different line of work (\cite{wangsketched}, \cite{oversketched}, \cite{newtonsketch}) uses Hessian sketching to solve a large-scale distributed learning problems.

\textit{Byzantine Robust Optimization:} In the seminal work of \cite{feng},  a generic framework of one shot median based robust learning has been proposed and analyzed in the distributed setting.  The issue of Byzantine failure is tackled by grouping the servers in batches and computing the median of batched servers in
\cite{chen} (the median of means algorithm). Later in \cite{dong,dong1}, co-ordinate wise median, trimmed mean and iterative filtering based algorithm have been proposed and
optimal statistical error rate is obtained. Also,  \cite{mhamdi2018hidden,Damaskinos} consider adversaries may steer convergence to bad local minimizers for non-convex optimization problems.
Byzantine resilience with gradient quantization has been addressed in the recent works of \cite{anima,ghosh2019communication}.  
\paragraph*{Organization:} In Section~\ref{sec:one}, we first analyze \textsf{COMRADE} with \emph{one round} of communication per iteration. We assume $\alpha =0$, and focus on the communication efficiency aspect only. Subsequently, in Section~\ref{sec:byz}, we make $\alpha \neq 0$, thereby addressing communication efficiency and Byzantine resilience simultaneously. Further, in Section~\ref{sec:compress} we augment a compression scheme along with the setting of Section~\ref{sec:byz}. Finally, in Section~\ref{sec:exp}, we validate our theoretical findings with experiments. Proofs of all theoretical results can be found in the supplementary material.

\paragraph*{Notation:} For a positive integer $r$, $[r]$ denotes the set $\{1,2,\ldots,r\}$. For a vector $v$, we use $\|v\|$ to denote the $\ell_2$ norm unless otherwise specified. For a matrix $X$, we denote $\|X\|_{2}$ denotes the operator norm, $\sigma_{max}(X)$ and $\sigma_{min}(X)$ denote the maximum and minimum singular value. Throughout the paper, we use $C,C_1,c,c_1$ to denote positive universal constants, whose value changes with instances.

\vspace{-3mm}
\section{Problem Formulation}
\label{sec:formulation}
\vspace{-2mm}
We begin with the standard statistical learning framework for empirical risk minimization, where the objective is to minimize the following loss function: 
\begin{align}
f(\w)= \frac{1}{n} \sum_{j=1}^n \ell_j(\w^T\x_j) + \frac{\lambda}{2}\|\w\|^2,   \label{prob}
\end{align}
where, the loss functions $\ell_j : \mathbb{R} \rightarrow \mathbb{R}$, $j \in [n]$ are \emph{convex, twice differentiable and  smooth}. Moreover, $\x_1,\x_2,\ldots ,\x_n \, \in \mathbb{R}^d$ denote the input feature vectors and $y_1,y_2,\ldots ,y_n \in \real $  denote the corresponding responses. Furthermore, we assume that the function $f$  is \emph{strongly convex}, implying the existence of a unique minimizer of \eqref{prob}. We denote this minimizer by $\w^*$. Note that the response  $\{y_j\}_{j=1}^n$ is captured by the corresponding loss function $\{\ell_j\}_{j=1}^n$. Some examples of $\ell_j$ are 
\begin{align*}
\text{logistic loss: }\,\, \ell_j(z_j) =\log(1- \exp(-z_jy_j)), \quad
\text{squared loss: }\,\, \ell_j(z_j) = \frac{1}{2}(z_j-y_j)^2  
\end{align*}
We consider the framework of distributed optimization with $m$ worker machines, where the feature vectors and the loss functions $(\x_1,\ell_1),\ldots,(\x_n,\ell_n) $ are partitioned homogeneously among them. Furthermore, we assume that $\alpha$
fraction of the worker machines are Byzantine for some $\alpha < \frac{1}{2}$.  The Byzantine machines, by nature, may send any arbitrary values to the center machine. Moreover, they can even collude with each other and plan malicious attacks with complete information of the learning algorithm.
\vspace{-3mm}
\section{\textsf{COMRADE} Can Communicate Less}\label{sec:one}
\vspace{-2mm}
We first present the Newton-type learning algorithm, namely \textsf{COMRADE} without any Byzantine workers, i.e., $\alpha =0$. It is formally given in Algorithm~\ref{alg:main_algo} (with $\beta =0$).  In each iteration of our algorithm, every worker machine computes the local Hessian and local gradient and sends the local second order update (which is the product of the inverse of the local Hessian and local gradient) to the center machine. The center machine aggregates the  updates from the worker machines by averaging them and updates the model parameter $\w$. Later the center machine broadcast the parameter $\w $ to all the worker machines.

In any iteration $t$, a standard Newton algorithm requires the computation of exact Hessian $(\bH_t)$ and gradient $(\g_t)$ of the loss function which can be written as
\vspace{-2mm}
\begin{align}
\g_t = \frac{1}{n}\sum_{i=1}^n\ell_j'(\w_t^{\top}\x_i)\x_i + \lambda\w_t, \,\,
\bH_t = \frac{1}{n}\sum_{i=1}^n\ell_j^{''}(\w_t^{\top}\x_i)\x_i\x_i^{\top} + \lambda\mathbf{I}. \label{globalgradhess}
\vspace{-2mm}
\end{align}
In a distributed set up, the exact  Hessian $(\bH_t)$ and gradient $(\g_t)$ can be computed in parallel in the following manner. In each iteration, the center machine  `broadcasts' the model parameter $\w_t$ to the worker machines and each worker machine computes its own local gradient and Hessian. Then the center machine can compute the exact gradient and exact Hessian by averaging the the local gradient vectors and local Hessian matrices. But for each worker machine the per iteration communication complexity is $\mathcal{O}(d)$ for the gradient computation and $\mathcal{O}(d^2$) for the Hessian computation. Using Algorithm~\ref{alg:main_algo}, we reduce the communication cost to only $\mathcal{O}(d)$ per iteration, which is the same as the first order methods. 

Each worker machine possess $s$ samples drawn uniformly from $\{ (\x_1,\ell_1),(\x_2,\ell_2),\ldots ,(\x_n,\ell_n)  \}$. By $S_i$, we denote the indices of the samples held by worker machine $i$. At any iteration $t$, the worker machine computes  the local Hessian $\bH_{i,t}$ and local gradient $\g_{i,t}$ as
\vspace{-2mm}
\begin{align}
\g_{i,t}  = \frac{1}{s}\sum_{i \in S_i}\ell_j'(\w_t^{\top}\x_i)\x_i + \lambda\w_t,  \quad
\bH_{i,t}= \frac{1}{s}\sum_{i \in S_i}\ell_j^{''}(\w_t^{\top}\x_i)\x_i\x_i^{\top} + \lambda\mathbf{I}. \label{localgradhess}
\end{align}
It is evident from the uniform sampling that $\mathbb{E}[\g_{i,t}]=\g_t $ and $ \mathbb{E}[\bH_{i,t}]=\bH_t$. The update direction from the worker machine is defined as $\hat{\p}_{i,t}= (\bH_{i,t})^{-1}\g_{i,t}$.  Each worker machine requires $O(sd^2)$ operations to compute the Hessian matrix $\bH_{i,t}$ and $O(d^3)$ operations to invert the matrix. In practice, the computational cost can be reduced by employing conjugate gradient method. The center machine computes the parameter update direction $\hat{\p}_t= \frac{1}{m}\sum_{i=1}^m \hat{\p}_{i,t}$.

 We show that given large enough sample in each worker machine ($s$ is large) and with incoherent data points (the information is spread out and not concentrated to a small number of sample data points), the local Hessian $\bH_{i,t}$ is  close to the global Hessian $\bH_t$ in spectral norm, and the local gradient $\g_{i,t}$ is close to the global gradient $\g_t$. Subsequently, we prove that the empirical average of the local updates acts as a good proxy for the global Newton update and achieves good convergence guarantee. 
 
\begin{algorithm}[t!]
  \caption{COMmunication-efficient and Robust Approximate Distributed nEwton (\textsf{COMRADE})}
  \begin{algorithmic}[1]
 \STATE  \textbf{Input:} Step size $\gamma$, parameter $\beta \ge 0$ 
 \STATE \textbf{Initialize:} Initial iterate $w_0 \in \mathbb{R}^d$ \\
  \FOR{$t=0,1, \ldots, T-1 $}
\STATE \underline{Central machine:} broadcasts $w_t$  \\
  \textbf{ for $ i \in [m]$ do in parallel}\\
  \STATE \underline{$i$-th worker machine:} 
    \begin{itemize}
    \item Non-Byzantine: Computes local gradient $\g_{i,t}$  and local Hessian $\bH_{i,t}$; sends $\hat{\p}_{i,t}=(\bH_{i,t})^{-1}\g_{i,t}$ to the central machine,
        \item Byzantine: Generates $\star$ (arbitrary), and sends it  to the center machine
         \end{itemize}
    \textbf{end for}
\STATE \underline{Center Machine:}
    \begin{itemize}
        \item Sort the worker machines in a non decreasing order according to norm of updates $\{ \hat{\p}_{i,t}\}_{i=1}^m$ from the local machines 
        \item Return the indices of the first $1-\beta$ fraction of machines as $\mathcal{U}_t$,
        \item Approximate Newton Update direction : $ \hat{\p}_t= \frac{1}{|\mathcal{U}_t|}\sum_{i\in \mathcal{U}_t} \hat{\p}_{i,t}$ \\
       \item Update model parameter: $w_{t+1} = w_t - \gamma \hat{\p}_t$.
   \end{itemize}
  \ENDFOR
  \end{algorithmic}\label{alg:main_algo}
\end{algorithm}

\subsection{Theoretical Guarantee}
\vspace{-2mm}
We define the matrix $\bA_t^{\top}= [\ma_1^{\top},\ldots, \ma_n^{\top}]\in \mathbb{R}^{d\times n}$ where $\ma_j = \sqrt{\ell''_j(\w^{\top}\x_j)}\,\x_j$. So the exact Hessian in equation~\eqref{globalgradhess} is  $\bH_t= \frac{1}{n}\bA_t^{\top}\bA_t + \lambda \mathbf{I}$. Also we define $\bB_t= [\mb_1,\ldots ,\mb_n] \in \mathbb{R}^{d\times n}$ where $\mb_i= \ell'_i(\w^T\x_i)\x_i$. So the exact gradient in equation~\eqref{globalgradhess} is $\g_t= \frac{1}{n}\bB_t\mathbf{1}+ \lambda\w_t$

\begin{definition}[Coherence of a Matrix] \label{def:cohe}
Let $\bA \in \mathbb{R}^{n \times d}$ be any matrix with $\bU \in \mathbb{R}^{n \times d}$ being its orthonormal basis (the left singular vectors). The row coherence of the matrix $\bA$ is defined as $\mu(\bA)=\frac{n}{d} \max_i \left\| \bu_i\right\|^2 \in  \left[1,\frac{n}{d} \right]$, where $\bu_i$ is the $i$th row of $\bU$.
\end{definition}

\begin{remark}
If the coherence of $\bA_t$ is small, it can be shown that the Hessian matrix can be approximated well via selecting a subset of rows. Note that this is a fairly common to use coherence condition as an approximation tool (see \cite{drineas2012fast,randnla,mahoneyrandomized})
\end{remark}
In the following, we assume that the Hessian matrix is  $L$-Lipschitz (see definition below), which is a standard assumption for the analysis of the second order method for general smooth loss function (as seen in \cite{giant},\cite{deter}).

\begin{assumption}\label{asm:hess}
The Hessian matrix of the loss function $f$ is $L$-Lipschitz continuous i.e. $\left\|\nabla^2 f(w)- \nabla^2 f(w') \right\|_2 \leq L \left\|w-w' \right\|$.
\end{assumption}


In the following theorem, we provide the convergence rate of \textsf{COMRADE} (with $\alpha=\beta =0$) in the terms of $\mathbf{\Delta}_{t}=\w_t-\w^*$. Also, we define $\kappa_t = \sigma_{max} (\bH_t)/\sigma_{min}(\bH_t)$ as the condition number of $\bH_t$, and hence $\kappa_t \geq 1$.

\begin{theorem}\label{thm:smooth}
Let $\mu \in \left[ 1, \frac{n}{d}\right]$ be the coherence of $\bA_t$ . Suppose $\gamma =1$ and $s \geq  \frac{3\mu d}{\eta^2}\log \frac{md}{\delta}$ for some $\eta,\delta \in (0,1)$.  Under Assumption ~\ref{asm:hess} , with probability exceeding $1-\delta$, we obtain 
\begin{align*}
  \| \mathbf{\Delta}_{t+1}\| \leq \max \{ \sqrt{\kappa_t(\frac{\zeta^2}{1-\zeta^2})}\| \mathbf{\Delta}_{t}\|, \frac{L}{\sigma_{min}(\bH_t) }\| \mathbf{\Delta}_{t}\|^2  \}+ \frac{ 2\epsilon}{\sqrt{\sigma_{min}(\bH_t)}},
\end{align*}
 where $ \zeta= \nu(\frac{\eta}{\sqrt{m}}+ \frac{\eta^2}{1-\eta})$,  $\nu= \frac{\sigma_{max}(\bA^{\top}\bA)}{\sigma_{max}(\bA^{\top}\bA)+n\lambda} \leq 1$, and 
\begin{align}
\epsilon =\frac{1}{1-\eta}\frac{1}{\sqrt{\sigma_{min}(\bH_t)}} (1+ \sqrt{2\ln (\frac{m}{\delta})})\sqrt{\frac{1}{s}}\max_i \|\mb_i\|. \label{eps}
\end{align}
\end{theorem}

\begin{remark}
It is well known that a distributed Newton method has linear-quadratic convergence rate. In Theorem~\ref{thm:smooth} the quadratic term comes from the standard analysis of Newton method.  The linear term (which is small)  arises owing to Hessian approximation. It gets smaller with better Hessian approximation (smaller $\eta$), and thus the above rate becomes quadratic one.  The small error floor arises due to the gradient approximation in the worker machines, which is essential for the one round of communication per iteration. The error floor is $\propto \frac{1}{\sqrt{s}}$ where $s$ is the number of samples in each worker machine. So for a sufficiently large $s$, the error floor becomes negligible. 
\end{remark}

\begin{remark}
The sample size in each worker machine is dependent on the coherence of the matrix $\bA_t$ and the dimension $d$ of the problem. Theoretically, the analysis is feasible for the case of $s \geq d$ (since we work with $\bH_{i,t}^{-1}$). However, when $s<d$, one can replace the inverse by a pseudo-inverse (modulo some changes in convergence rate). 
\end{remark}
\section{\textsf{COMRADE} Can Resist Byzantine Workers} \label{sec:byz}
In this section, we analyze \textsf{COMRADE} with Byzantine workers. We assume that $\alpha (< 1/2)$  fraction of worker machines are Byzantine. We define the set of Byzantine worker machines by $\cB$ and the set of the good (non-Byzantine)  machines by $\cM$. \textsf{COMRADE} employs a `norm based thresholding' scheme on the local Hessian inverse and gradient product to tackle the Byzantine workers.

In the $t$-th iteration, the center machine outputs a set  $\cU_t$ with  $\left| \cU_t \right|= (1-\beta)m$, consisting the indices of the worker machines with smallest norm. Hence, we `trim' the worker machines that may try to diverge the learning algorithm.  We denote the set of trimmed machines as $\cT_t$. Moreover, we take $\beta > \alpha$ to ensure at least one good machine falls in $\cT_t$. This condition helps us to control the Byzantine worker machines. Finally, the update is given by $\hat{\p}_t= \frac{1}{|\cU_t|}\sum_{i\in \cU_t}\hat{\p}_{i,t} $. 
We define:
\small
\begin{align}
  & \epsilon_{byz}^2 = [3 (\frac{1-\alpha}{1-\beta})^2 +4\kappa_t (\frac{\alpha}{1-\beta})^2 ]\epsilon^2, \label{gbyz}\\
   & \zeta^2_{byz}  =2 (\frac{1-\alpha}{1-\beta} )^2(\frac{\nu}{1-\eta})^2 + \nu^2 (\frac{1-\alpha}{1-\beta})^2 (\frac{\eta}{\sqrt{(1-\alpha)m}}+ \frac{\eta^2}{1-\eta} )^2 + 4\kappa_t (\frac{\alpha}{1-\beta})^2 [2+(\frac{\nu}{1-\eta})^2 ]. \label{alpbyz} 
\end{align}
\normalsize

 $\epsilon$ is defined in \eqref{eps}, $\nu= \frac{\sigma_{max}(\bA^T\bA)}{\sigma_{max}(\bA^T\bA)+n\lambda}$ and $\kappa_t$ is the condition number of the exact Hessian $\bH_t$. 

\begin{theorem}\label{thm:byzsmooth}
Let $\mu \in \left[ 1, \frac{n}{d}\right]$ be the coherence of $\bA_t$ . Suppose $\gamma =1$ and $s \geq  \frac{3\mu d}{\eta^2}\log \frac{md}{\delta}$ for some $\eta,\delta \in (0,1)$. For $0\leq\alpha < \beta <1/2$  , under Assumption ~\ref{asm:hess} , with probability exceeding $1-\delta$, Algorithm~\ref{alg:main_algo} yields 
\begin{align*}
  \| \mathbf{\Delta}_{t+1}\| \leq \max \{ \sqrt{\kappa_t(\frac{\zeta_{byz}^2}{1-\zeta_{byz}^2} )}\| \mathbf{\Delta}_{t}\|, \frac{L}{\sigma_{min}(\bH_t) }\| \mathbf{\Delta}_{t}\|^2  \}+ \frac{ 2\epsilon_{byz} }{\sqrt{\sigma_{min}(\bH_t)}},
\end{align*}
where $\zeta_{byz}$ and $\epsilon_{byz}$ are defined in equations~\eqref{gbyz} and \eqref{alpbyz} respectively.
 \end{theorem}
The remarks of Section~\ref{sec:one} is also applicable here. On top of that, we have the following remarks:
\begin{remark}
Compared to the convergence rate of Theorem~\ref{thm:smooth}, the rate here remains order-wise same even with Byzantine robustness. The coefficient of the quadratic term remains unchanged but the linear rate and the error floor suffers a little bit (by a small constant factor).
\end{remark}
\begin{remark}
Note that for Theorem~\ref{thm:byzsmooth} to hold, we require $\alpha \sim 1/\sqrt{\kappa_t}$  for all $t$. In cases where $\kappa_t$ is large, this can impose a stricter condition on $\alpha$. However, we conjecture that this dependence can be improved via applying a more intricate (and perhaps computation heavy) Byzantine resilience algorithm. In this work, we kept the Byzantine resilience scheme simple at the expense of this condition on $\alpha$.
\end{remark}

\section{\textsf{COMRADE} Can Communicate Even Less and Resist Byzantine Workers}\label{sec:compress}
In Section~\ref{sec:one} we analyze \textsf{COMRADE} with an additional feature. We let the worker machines further reduce the communication cost by applying a generic class of $\rho$-approximate compressor \cite{errorfeed} on the parameter update of Algorithm~\ref{alg:main_algo}. We first define the class of $\rho$-approximate compressor:

\begin{definition} \label{def:comp}
An operator $\cQ: \mathbb{R}^d \rightarrow \mathbb{R}^d$ is defined as $\rho$-approximate compressor on a set $S \subset \mathbb{R}^d$ if, $\forall x \in S$,
$
\left\|\cQ(x)-x  \right\|^2 \leq (1-\rho)\left\| x  \right\|^2
$, 
where $\rho \in [0,1]$ is the compression factor. 
\end{definition}

The above definition can be extended for  any randomized operator $\cQ$ satisfying 
$
\mathbb{E}(\left\|\cQ(x)-x  \right\|^2 ) \leq (1-\rho)\left\| x  \right\|^2
$,
for all $\forall x \in S$. The expectation is taken over the randomization of the operator. Notice that $\rho=1$ implies that $\cQ(x)=x$ (no compression). Examples of $\rho$-approximate compressor include QSGD \cite{qsgd}, $\ell_1$-QSGD \cite{errorfeed}, top$_k$ sparsification and rand$_k$ \cite{stich2018sparsified}.

Worker machine $i$ computes the product of local Hessian inverse inverse and local gradient and then apply $\rho$-approximate compressor to obtain $\mathcal{Q}(\bH_{i,t}^{-1} \g_{i,t})$; and finally sends this compressed vector to the center. The Byzantine resilience subroutine remains the same--except, instead of sorting with respect to $ \|\bH_{i,t}^{-1} \g_{i,t}\|$, the center machine now sorts according to $\|\mathcal{Q}(\bH_{i,t}^{-1} \g_{i,t})\|$. The center machine aggregates the compressed updates by averaging $\cQ(\hat{\p})= \frac{1}{|\cU_t|}\sum_{i \in \cU_t} \cQ(\hat{\p}_{i,t})$, and take the next step as $w_{t+1} = w_t - \gamma  \cQ(\hat{\p})$.

Recall the definition of $\epsilon$ from \eqref{eps}. We also use the following notation : $\zeta^2_{\cM} = \nu(\frac{\eta}{\sqrt{(1-\alpha)m}}+ \frac{\eta^2}{1-\eta}),  \zeta_{1}= \frac{\nu}{1-\eta}$ and $\nu= \frac{\sigma_{max}(\bA^T\bA)}{\sigma_{max}(\bA^T\bA)+n\lambda}$. Furthermore, we define the following:
\small
\begin{align}
    \epsilon_{comp,byz}^2 &= [3 (\frac{1-\alpha}{1-\beta} )^2 +4\kappa_t (\frac{\alpha}{1-\beta})^2 ] (1+ \kappa (1-\rho))\epsilon^2 \label{ceps} \\
   \zeta^2_{comp,byz} & =2(\frac{1-\alpha}{1-\beta})^2 (\zeta^2_1 + \kappa_t (1-\rho)((1+\zeta^2_1) ) + (\frac{1-\alpha}{1-\beta})^2 (\zeta^2_{\cM} + \kappa_t (1-\rho)((1+\zeta^2_1)) \nonumber \\
   &+ 4\kappa _t (\frac{\alpha}{1-\beta})^2 (2+(\zeta^2_1 + \kappa_t (1-\rho)((1+\zeta^2_1))) \label{calpha}
\end{align}
\normalsize

\begin{theorem}\label{thm:compbyzsmooth}
Let $\mu \in \left[ 1, \frac{n}{d}\right]$ be the coherence of $\bA_t$ . Let $\gamma =1$ and $s \geq  \frac{3\mu d}{\eta^2}\log \frac{md}{\delta}$ for some $\eta,\delta \in (0,1)$. For $0\le \alpha < \beta <1/2$, under Assumption ~\ref{asm:hess}  and with $\cQ$ being the $\rho$-approximate compressor, with probability exceeding $1-\delta$, we obtain
\begin{align*}
 \| \mathbf{\Delta}_{t+1}\| \leq \max \{ \sqrt{\kappa_t (\frac{\zeta^2_{comp,byz}}{1-\zeta^2_{comp,byz}})}\| \mathbf{\Delta}_{t}\|, \frac{L}{\sigma_{min}(\bH_t) }\| \mathbf{\Delta}_{t}\|^2  \}+ \frac{ \epsilon_{comp,byz}}{\sqrt{\sigma_{min}(\bH_t)}}
\end{align*}
where $\epsilon_{comp,byz}$ and $\zeta_{comp,byz}$ are given in equations~\eqref{ceps} and \eqref{calpha} respectively.
\end{theorem}

\begin{remark}
 With no compression ($\rho=1$) we get back the convergence guarantee of Theorem~\ref{thm:byzsmooth}.
\end{remark}
\begin{remark}
Note that even with compression, we retain the linear-quadratic rate of convergence of \textsf{COMRADE}. The constants are affected by a $\rho$-dependent term.
\end{remark}

\begin{figure}[t!]%
    \centering
    \subfloat[w5a]{{\includegraphics[height = 3cm,width=3.5cm]{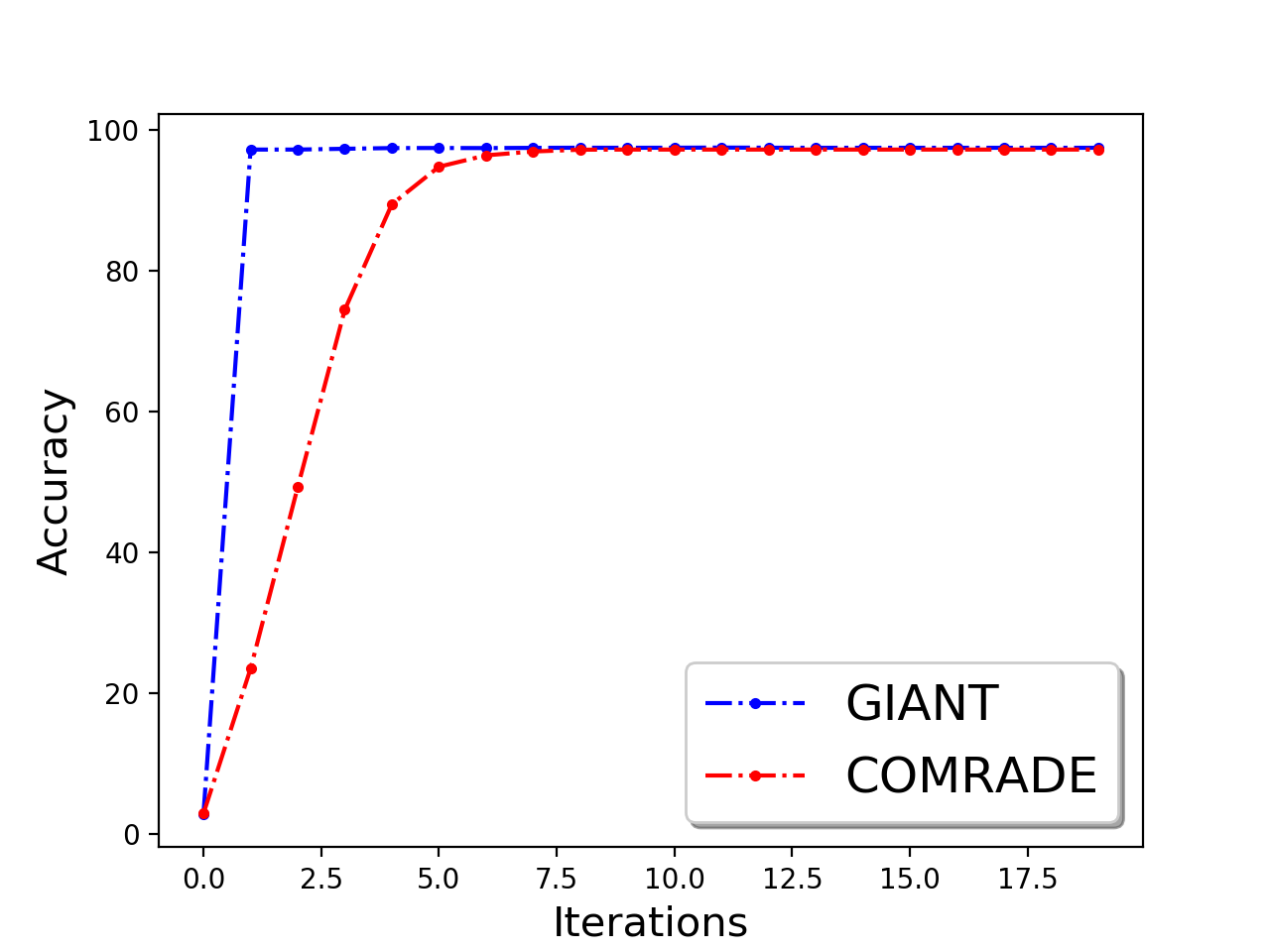} }}%
    \subfloat[a9a]{{\includegraphics[height = 3cm,width=3.5cm]{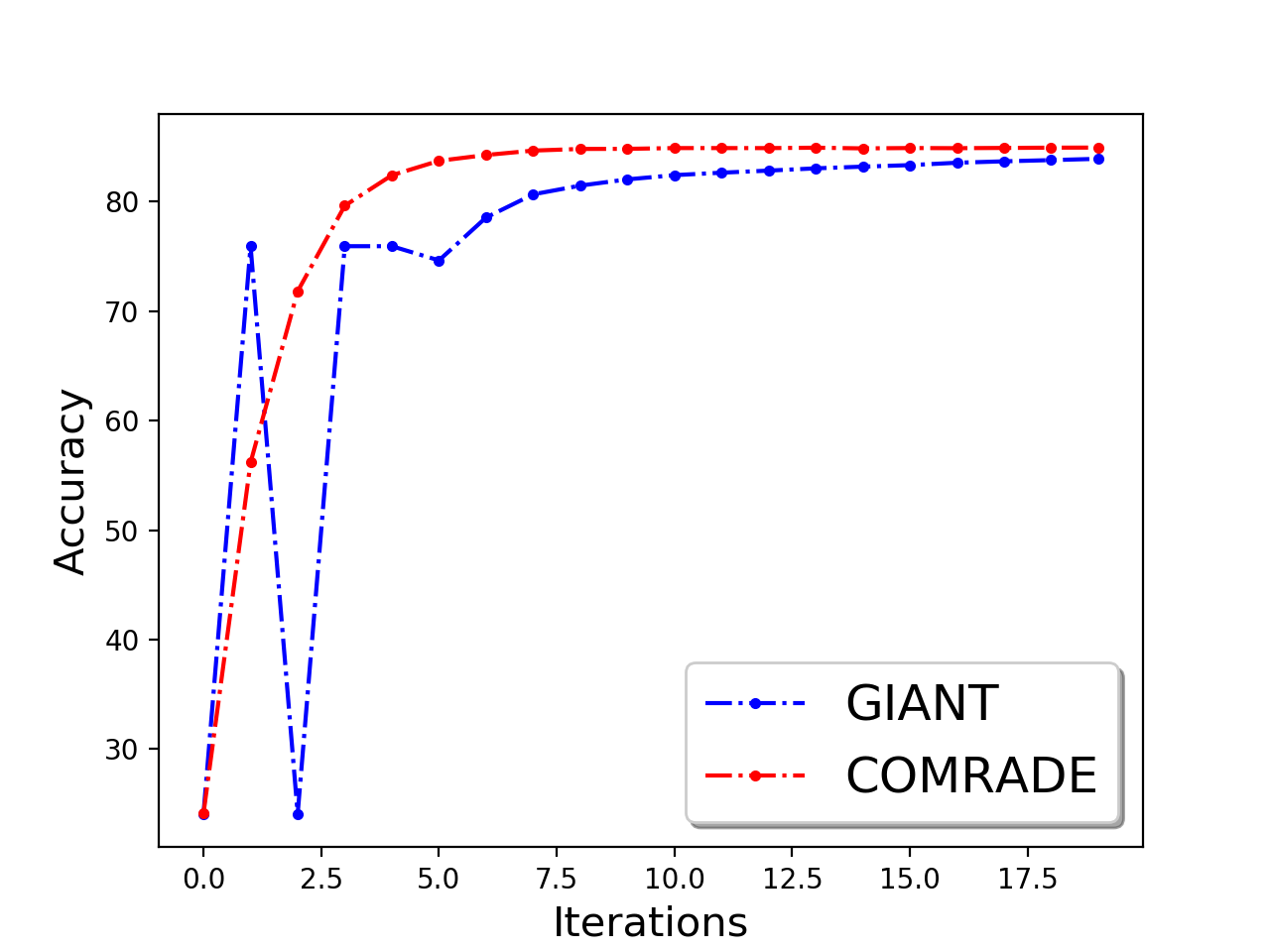} }}%
    \subfloat[Epsilon]{{\includegraphics[height = 3cm,width=3.5cm]{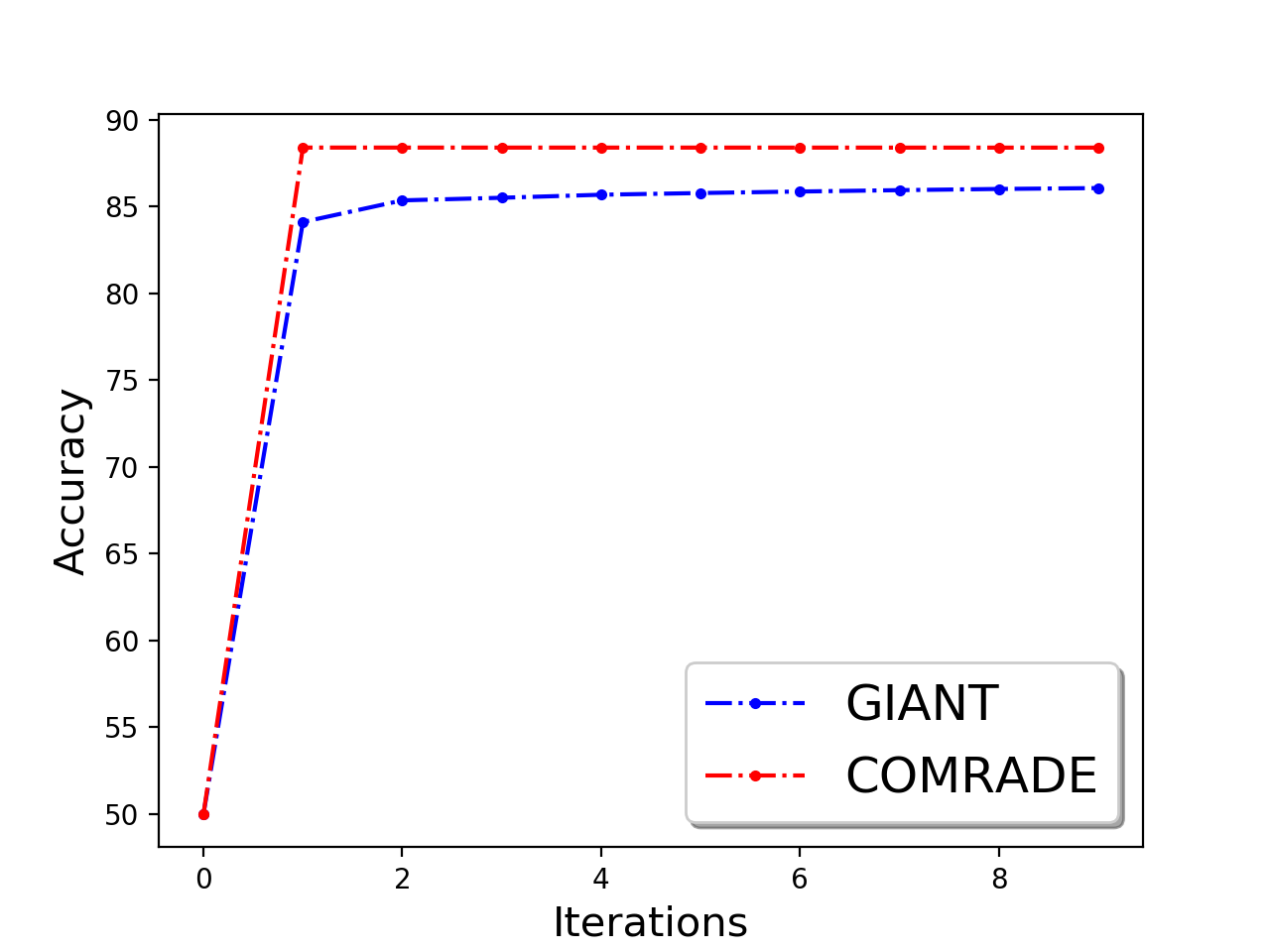} }}%
    \subfloat[covtype]{{\includegraphics[height = 3cm,width=3.5cm]{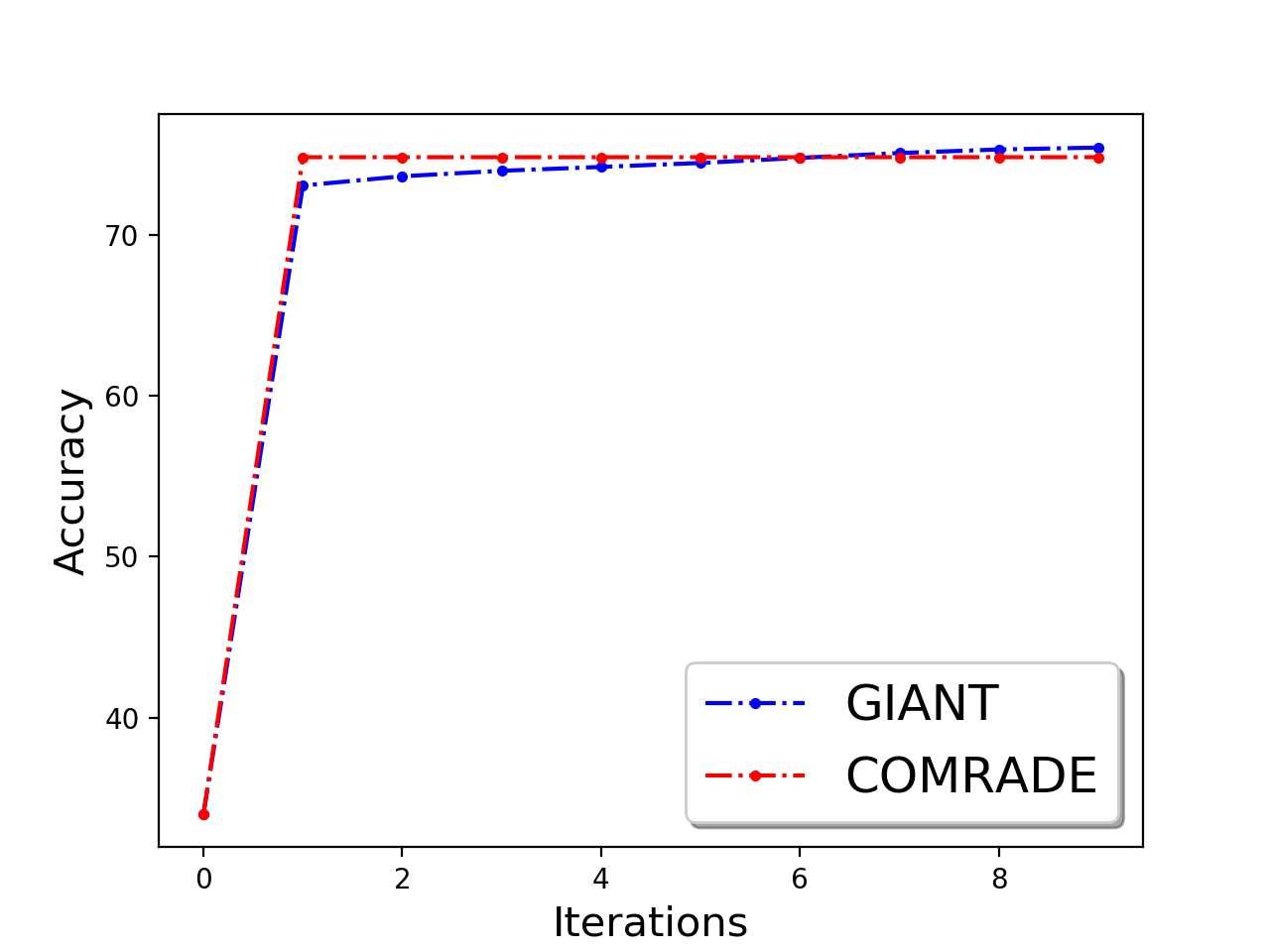} }}%
\vspace{-10pt}
    \subfloat[GIANT `flipped' attack]{{\includegraphics[height = 3cm,width=3.5cm]{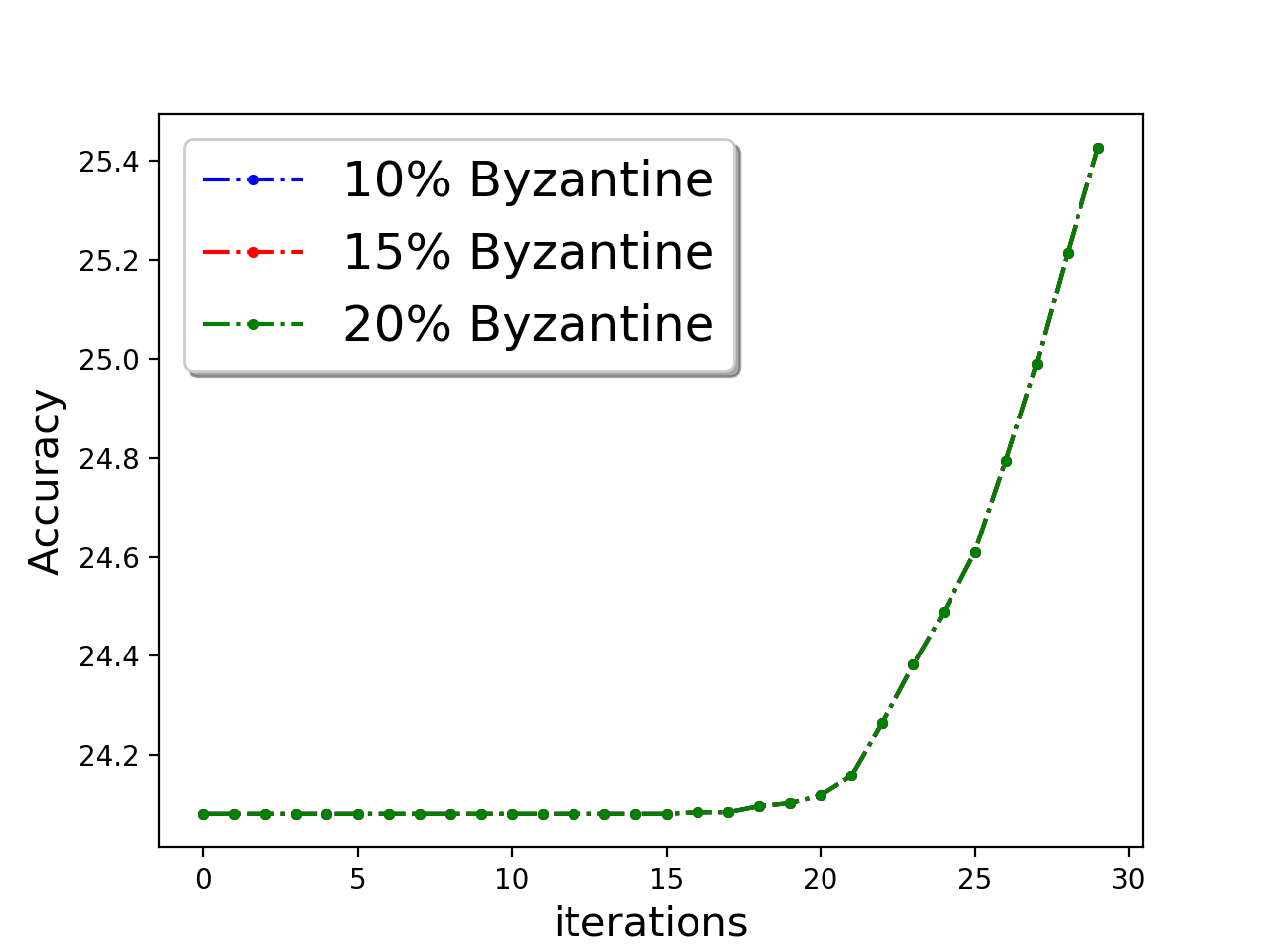} }}%
    \subfloat[GIANT `negative' attack]{{\includegraphics[height = 3cm,width=3.5cm]{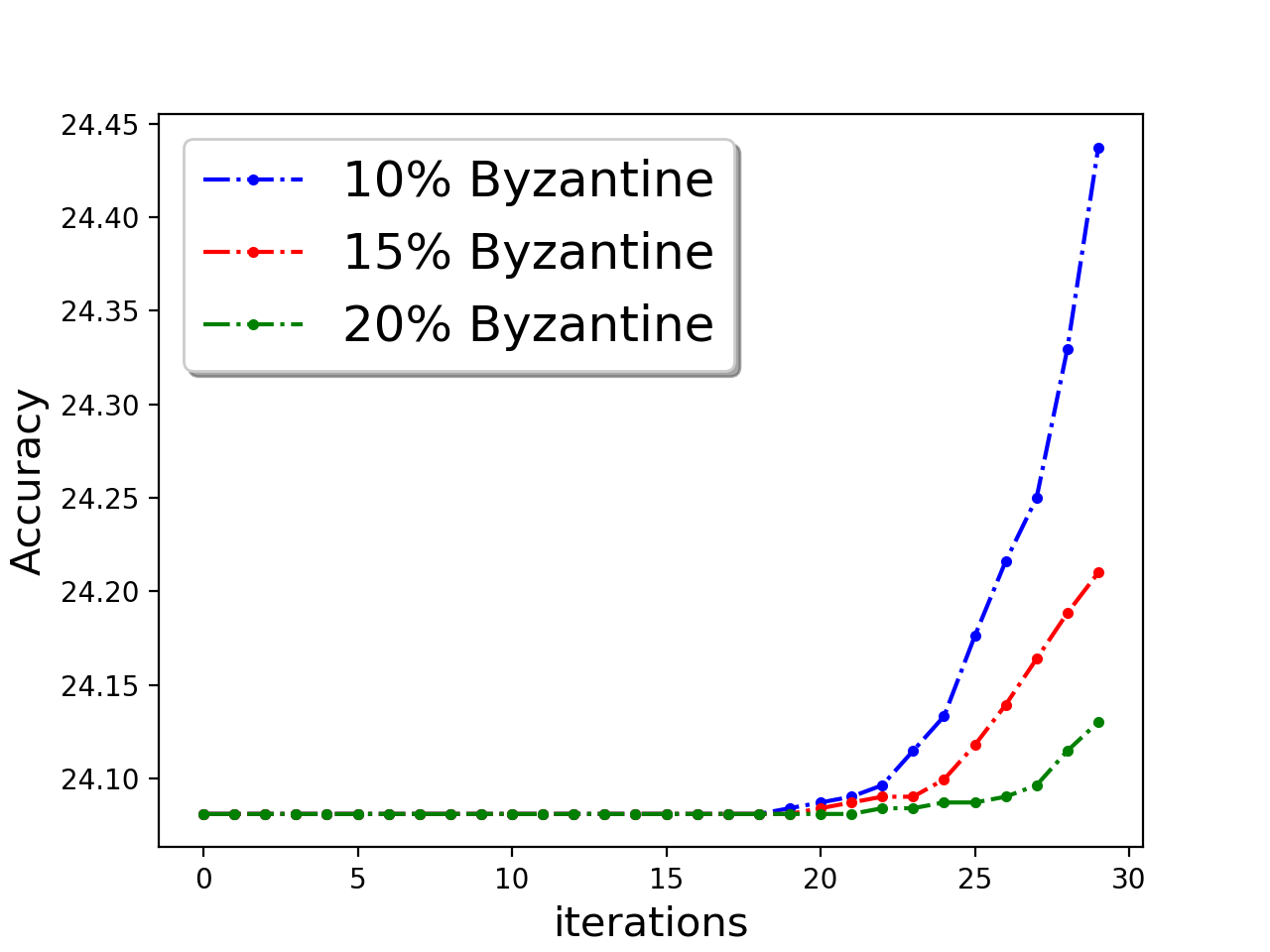} }}%
    \subfloat[Robust GIANT]{{\includegraphics[height = 3cm,width=3.5cm]{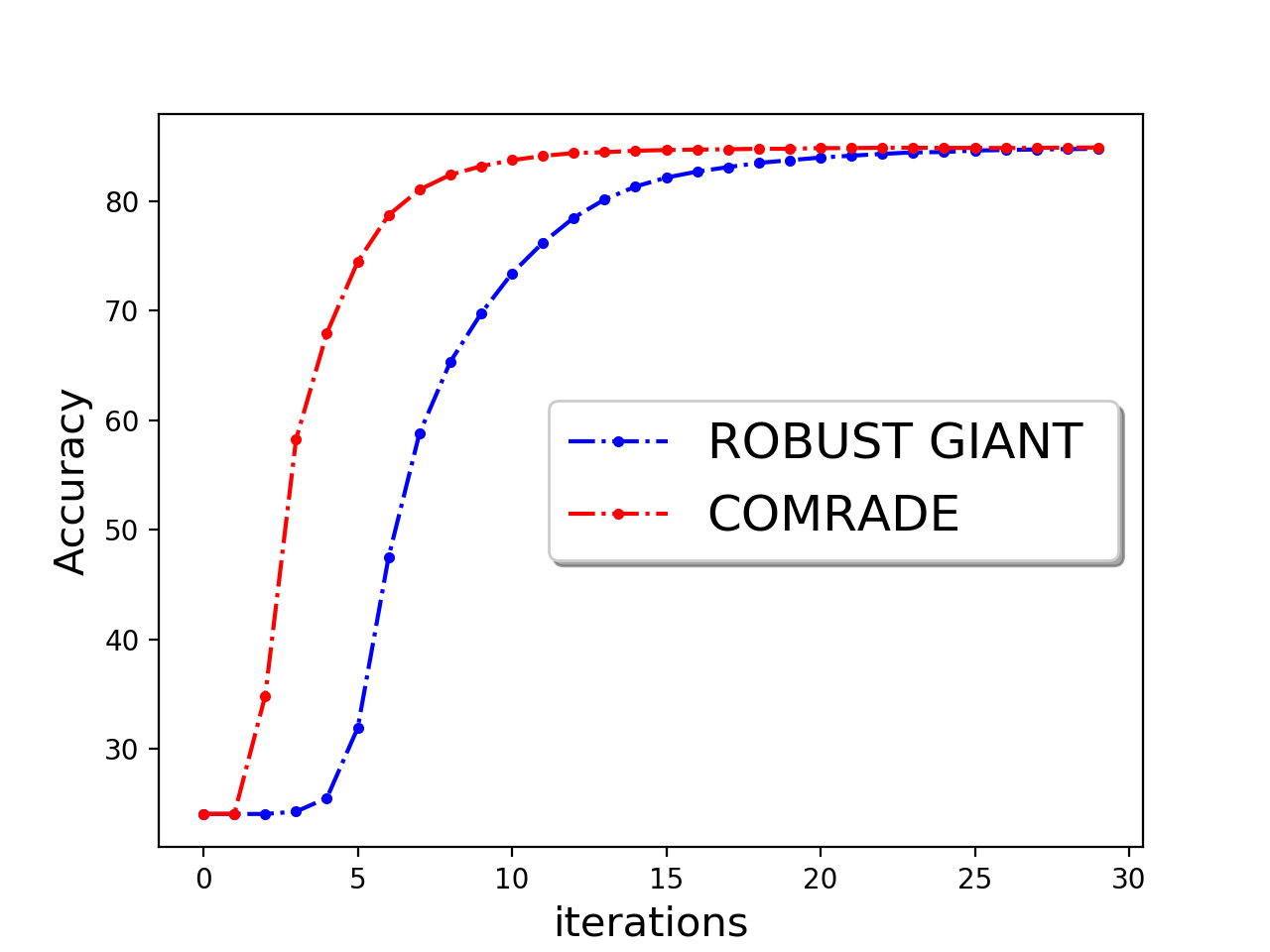} }}%
    \subfloat[Robust GIANT]{{\includegraphics[height = 3cm,width=3.5cm]{plotdata/comp/giant_a9a_det_thres.png} }}%

    \caption{(First row) Comparison of training accuracy between \textsf{COMRADE}(Algorithm~\ref{alg:main_algo}) and GIANT \cite{giant} with  (a) w5a (b) a9a  (c) Epsilon (d) Covtype dataset.  (Second row) Training accuracy of (e) GIANT for `flipped label' and (f) `negative update' attack; and comparison of Robust GIANT and \textsf{COMRADE} with a9a dataset for (g) `flipped label' and (h) `negative update' attack. }%
    \label{fig:comparision}%
\end{figure}
\section{Experimental Results} \label{sec:exp}

In this section we validate our algorithm, \textsf{COMRADE} in Byzantine  and non-Byzantine setup  on synthetically generated and benchmark LIBSVM \cite{libsvm} data-set. The experiments focus on the standard logistic regression problem. The logistic regression objective is defined as 
$\frac{1}{n} \sum_{i=1}^n \log \left( 1 + \exp(-y_i\x_i^{\top}\w ) \right) + \frac{\lambda}{2n}\| \w\|^2$, where $\w \in \mathbb{R}^d$ is the parameter, $\{\x_i\}_{i=1}^n \in \mathbb{R}^d$ are the feature data  and $\{ y_i\}_{i=1}^n \in \{0,1\}$ are the corresponding labels. We use `mpi4py' package   in   distributed computing framework (swarm2) at the University of Massachusetts Amherst \cite{swarm2}.
We choose `a9a' ($d=123 ,n\approx 32$K), `w5a' ($d=300,n\approx 10k$), `Epsilon' ($d=2000, n=0.4$M) and `covtype.binary' ($d=54 ,n\approx0.5$M)  classification datasets and partition the data in $20$ different worker machines.  In the experiments, we choose two types of Byzantine attacks : (1). `flipped label'-attack where (for binary classification) the Byzantine worker machines flip the labels of the data, thus making the model learn with wrong labels,  and (2). `negative update attack' where the Byzantine worker machines compute the local update ($\hat{\p}_{i}$) and communicate  $-c\times \hat{\p}_{i}$ with $c \in (0,1)$ making the updates  to be opposite of actual direction. We choose $\beta= \alpha +\frac{2}{m}$. 

\begin{figure}[h!]%
    \centering
   \subfloat[w5a]{{\includegraphics[height = 3cm,width=3.5cm]{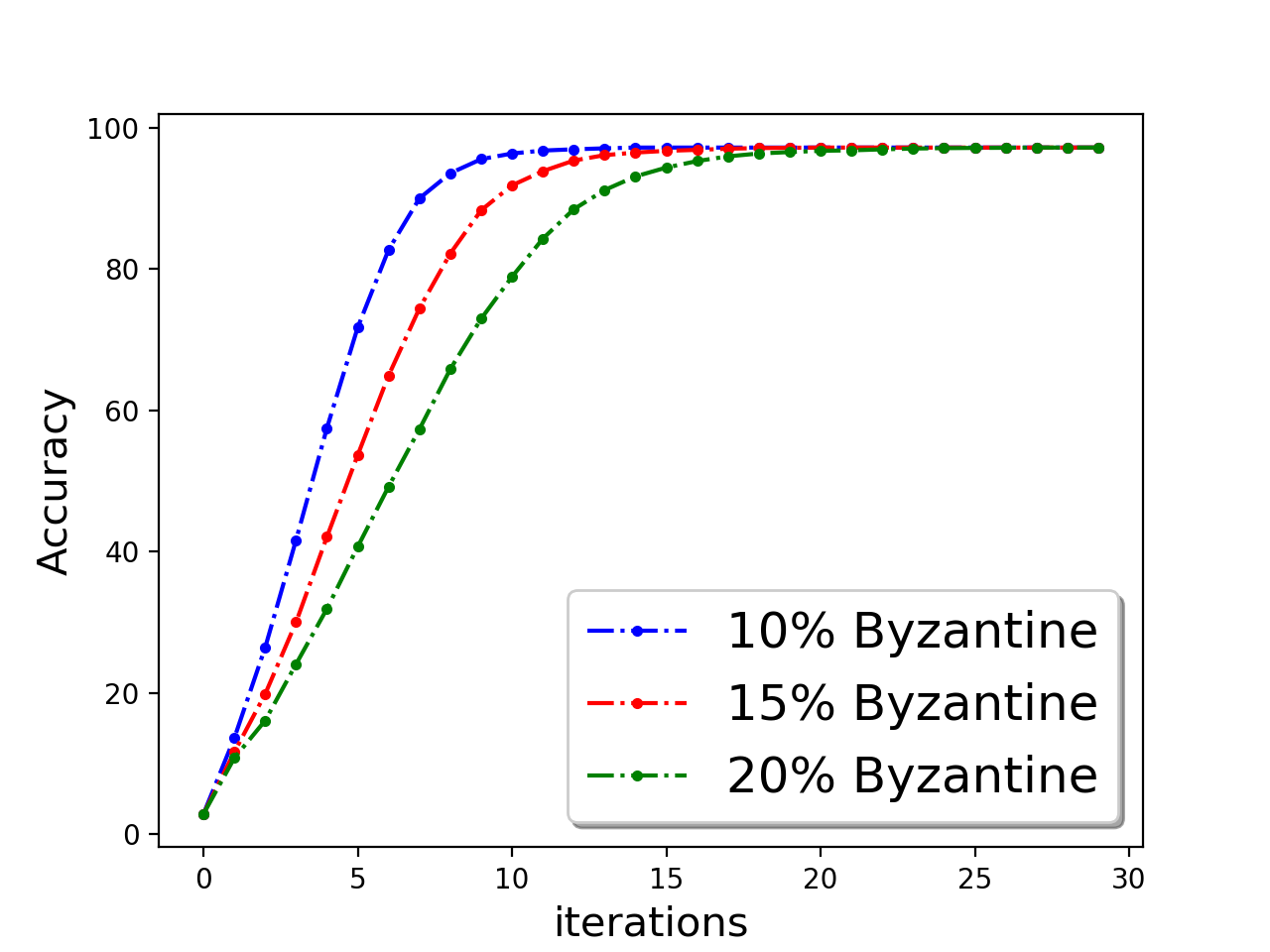} }}%
    \subfloat[a9a]{{\includegraphics[height = 3cm,width=3.5cm]{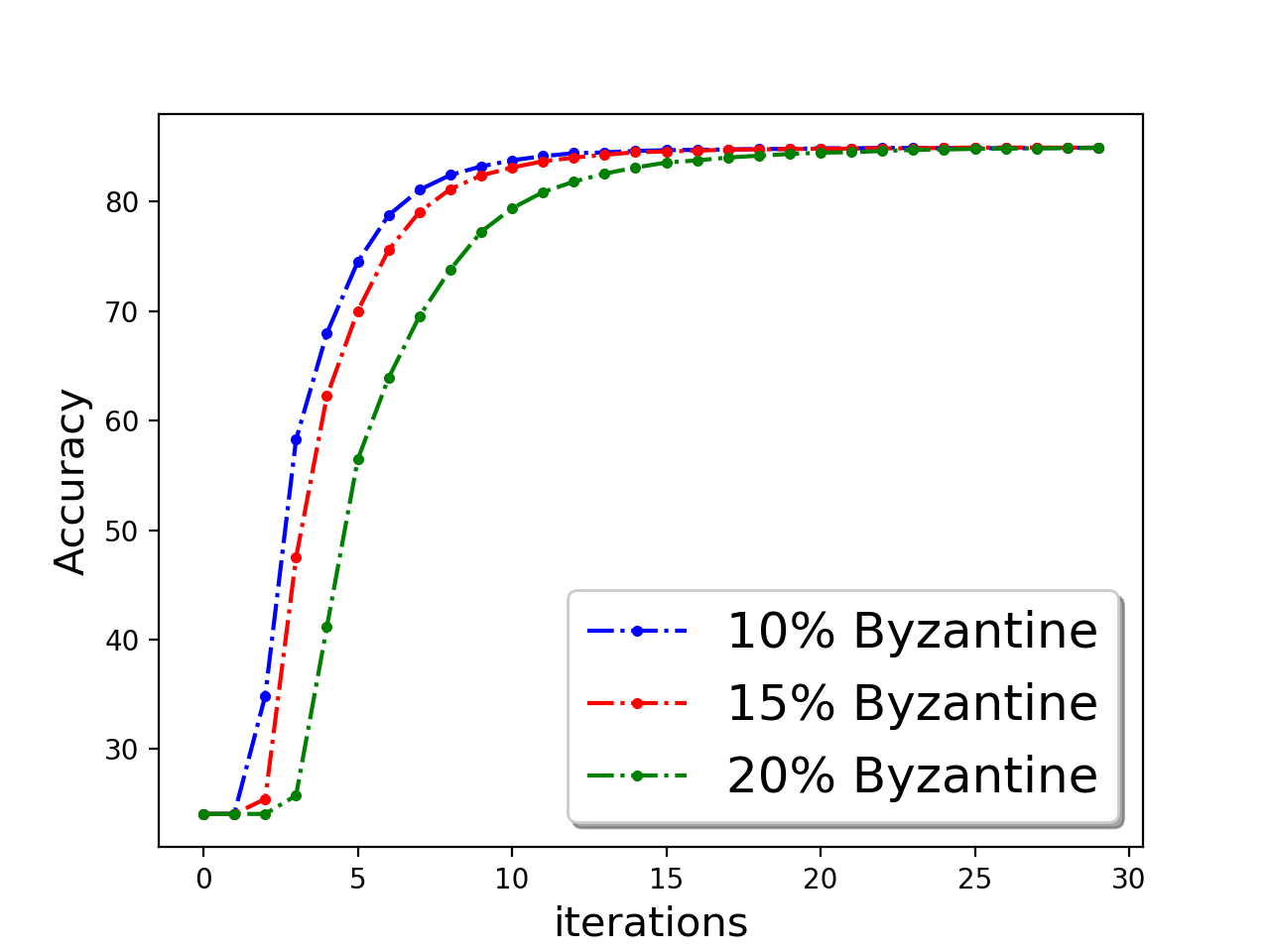} }}%
    \subfloat[Epsilon]{{\includegraphics[height = 3cm,width=3.5cm]{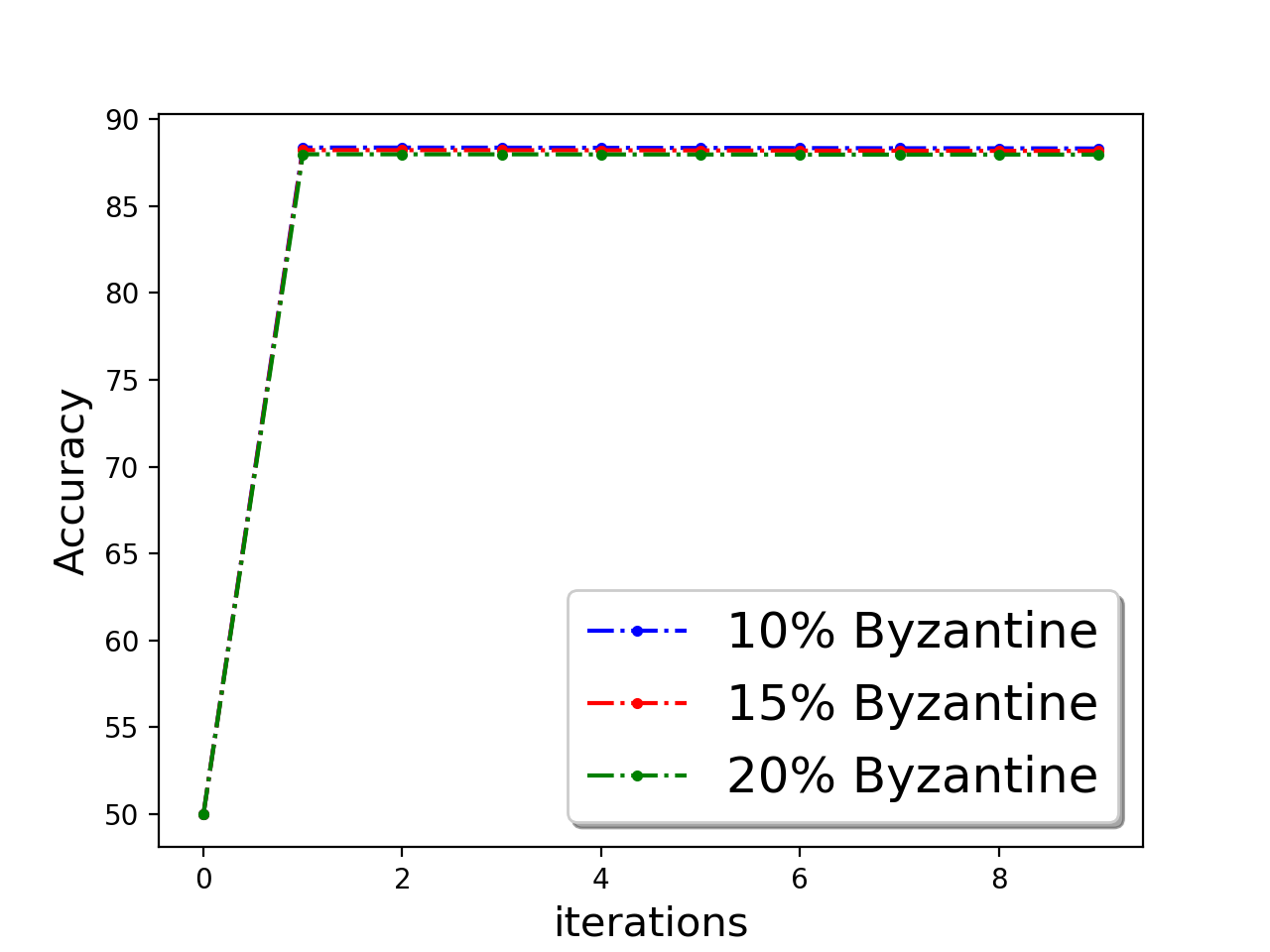} }}%
    \subfloat[covtype]{{\includegraphics[height = 3cm,width=3.3cm]{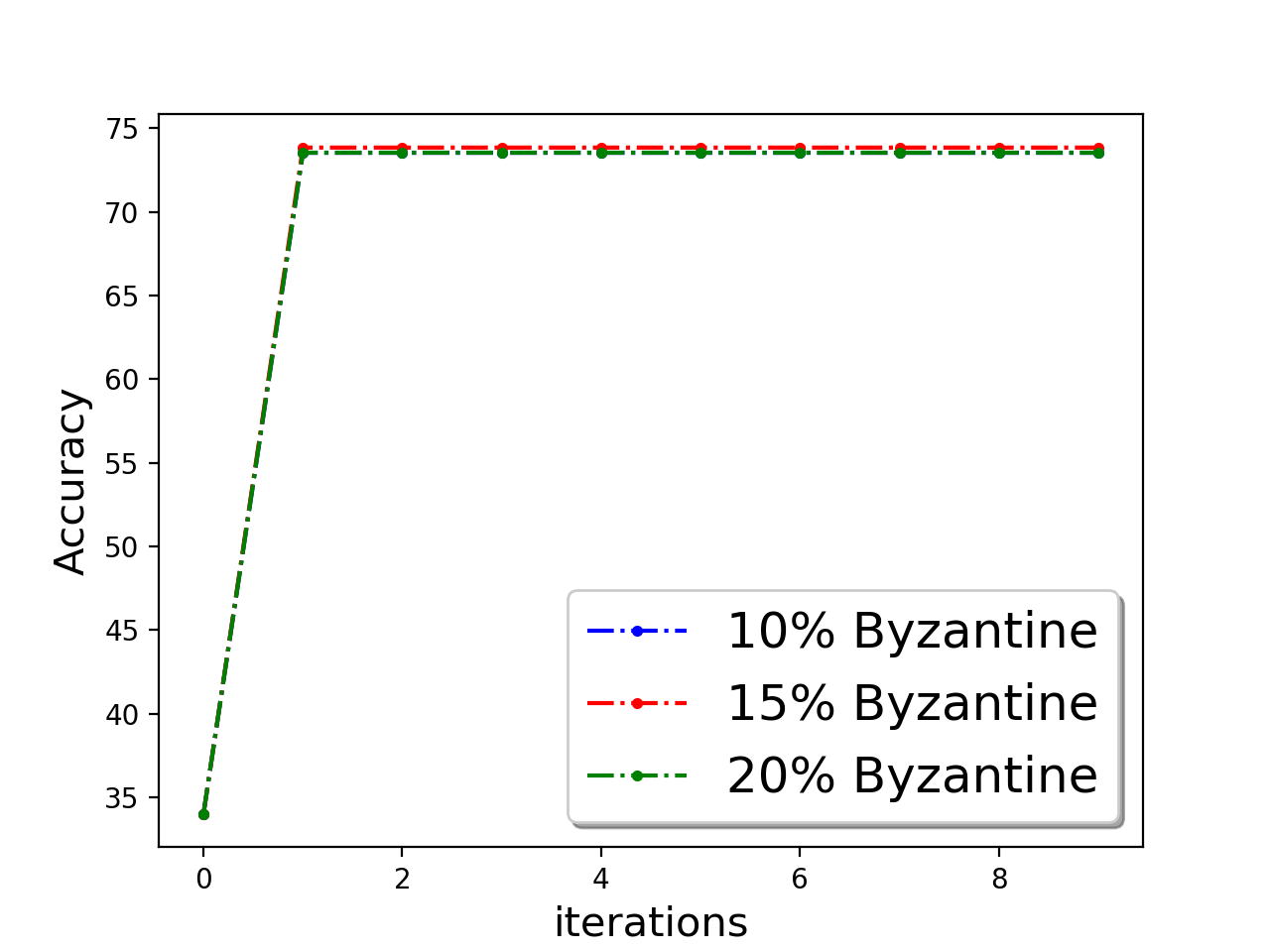} }}%
\vspace{-10pt}
  \subfloat[w5a]{{\includegraphics[height = 3cm,width=3.5cm]{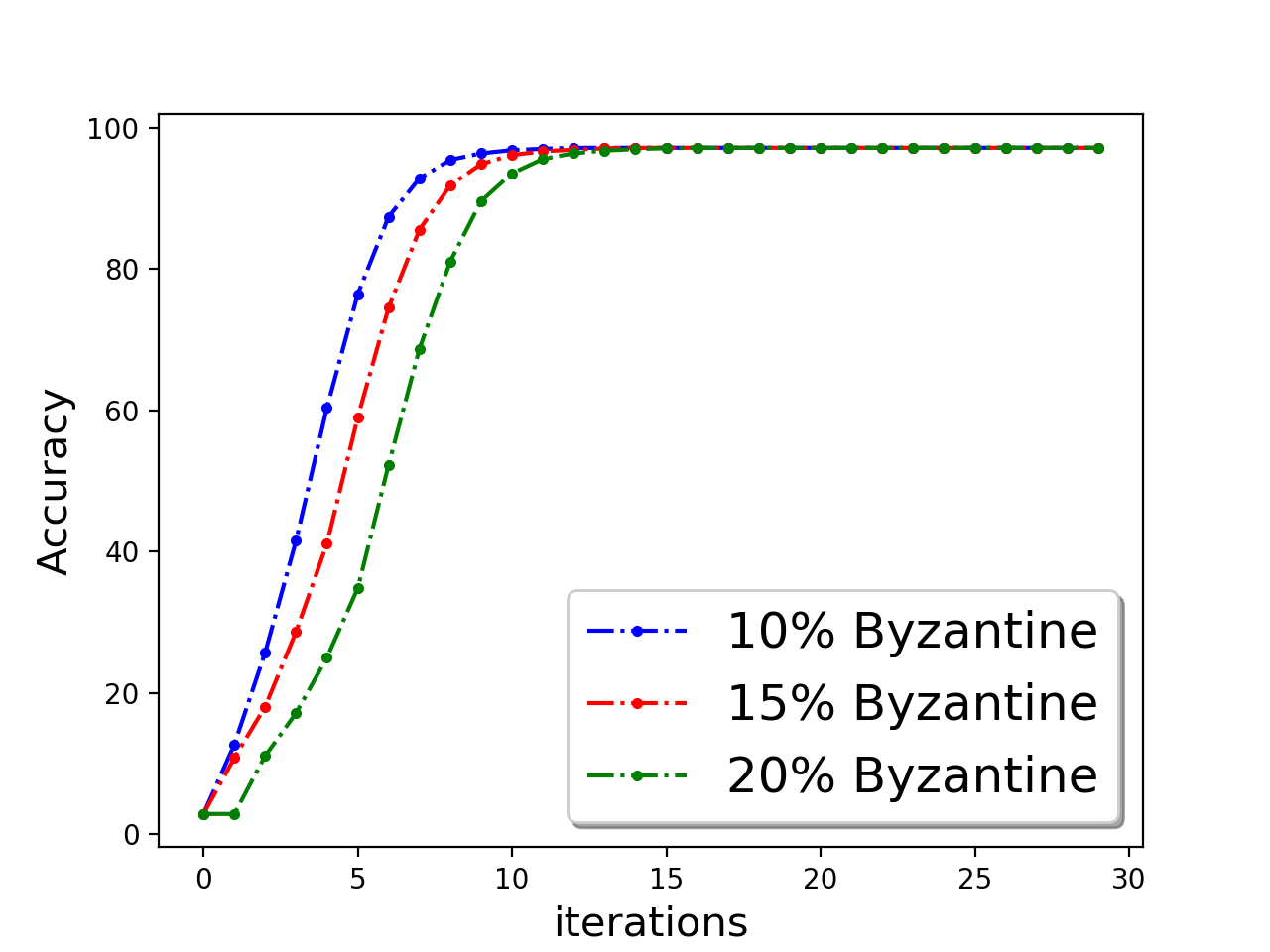} }}%
    \subfloat[a9a]{{\includegraphics[height = 3cm,width=3.5cm]{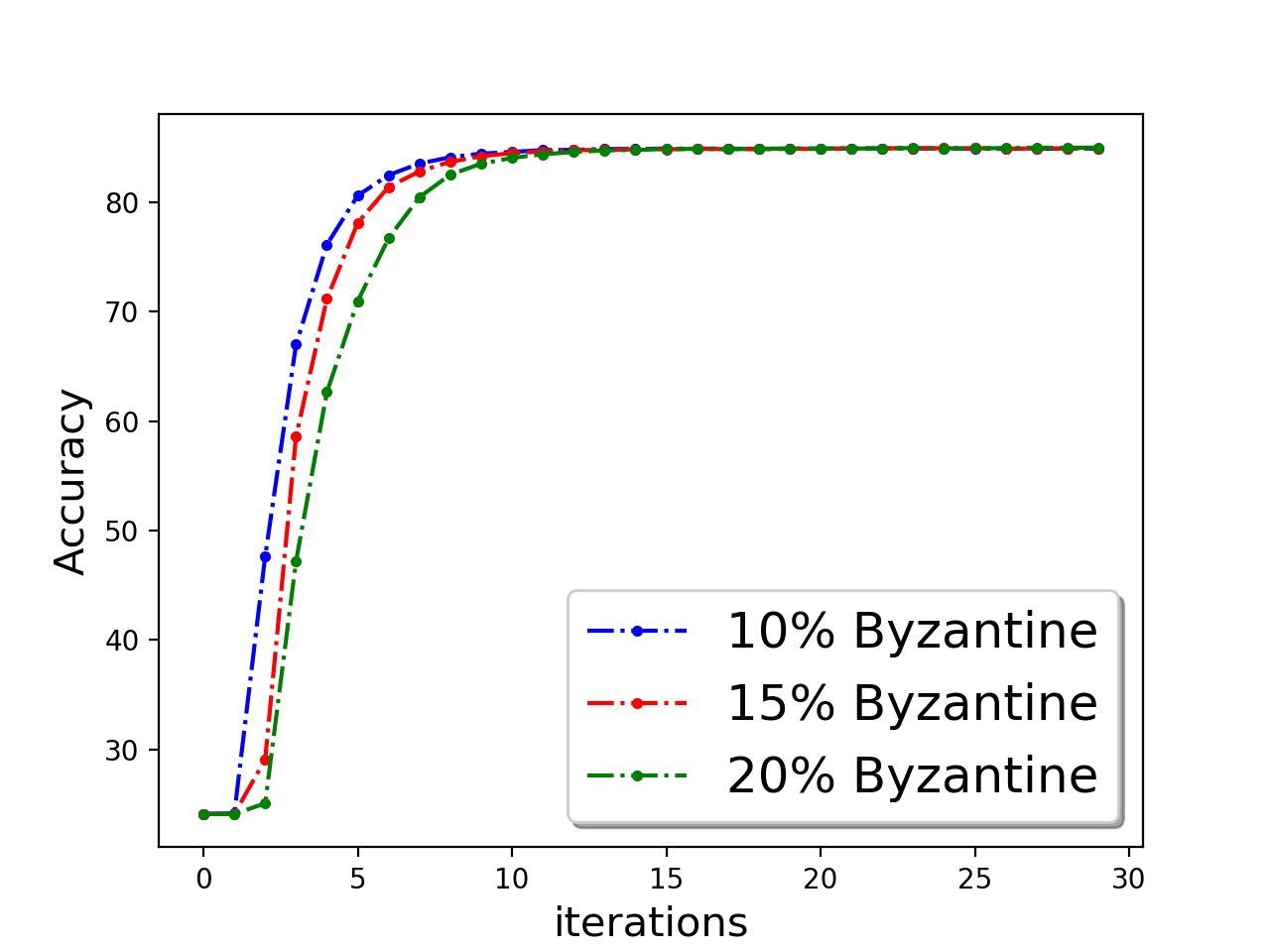} }}%
    \subfloat[Epsilon]{{\includegraphics[height = 3cm,width=3.5cm]{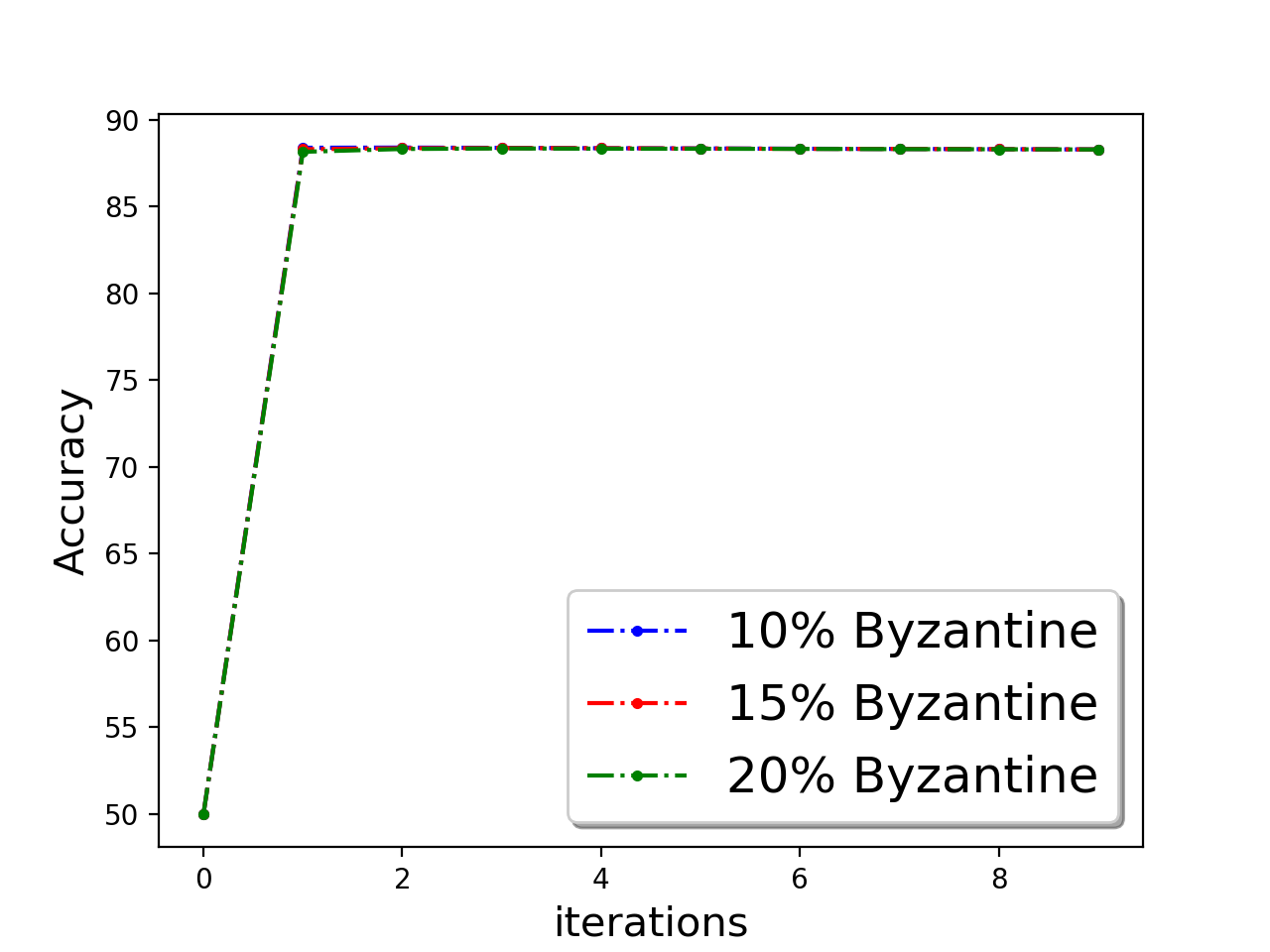} }}%
    \subfloat[covtype]{{\includegraphics[height = 3cm,width=3.3cm]{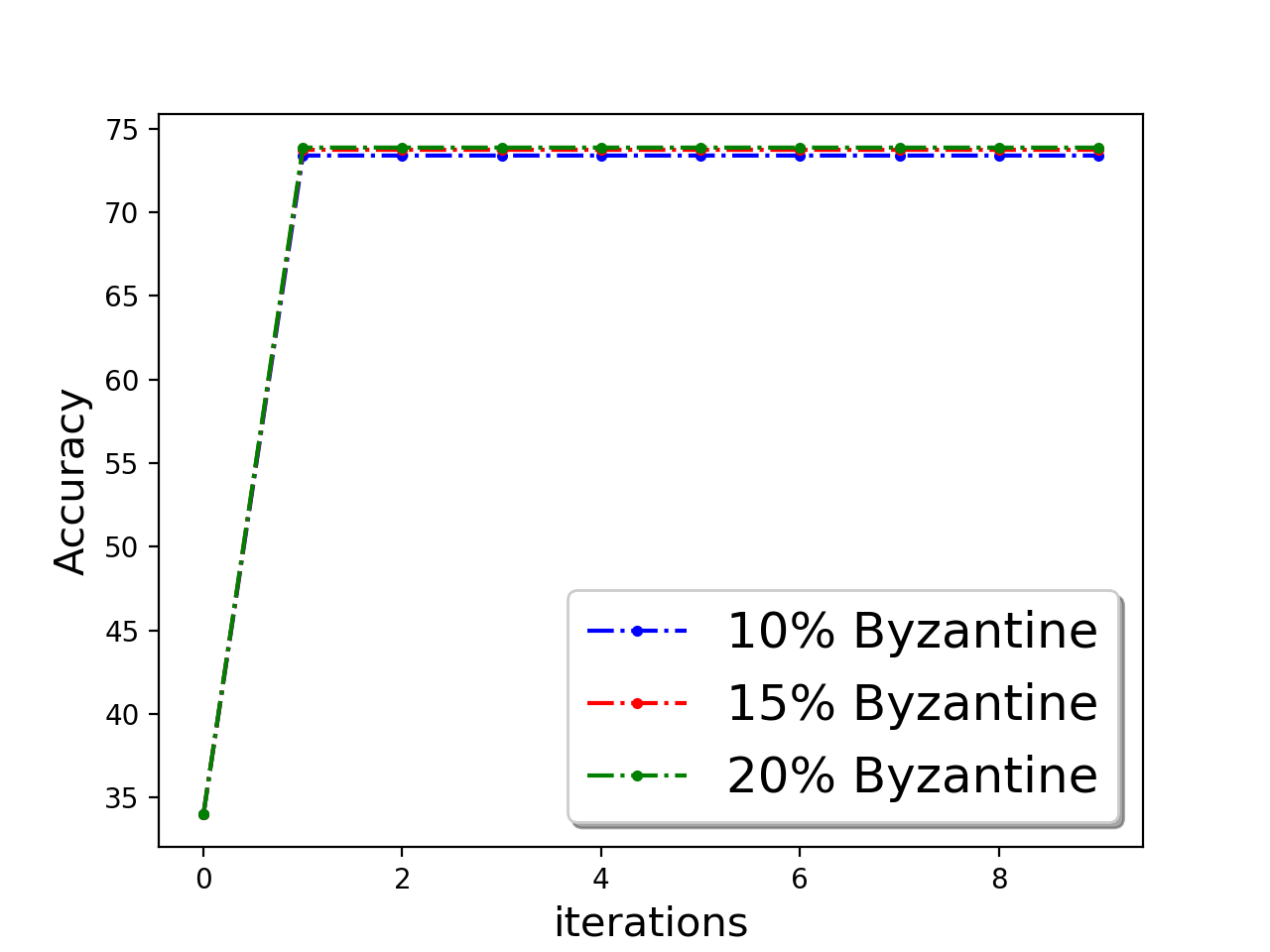} }}%
\vspace{-10pt}
    \subfloat[w5a `flipped']{{\includegraphics[height = 3cm,width=3.5cm]{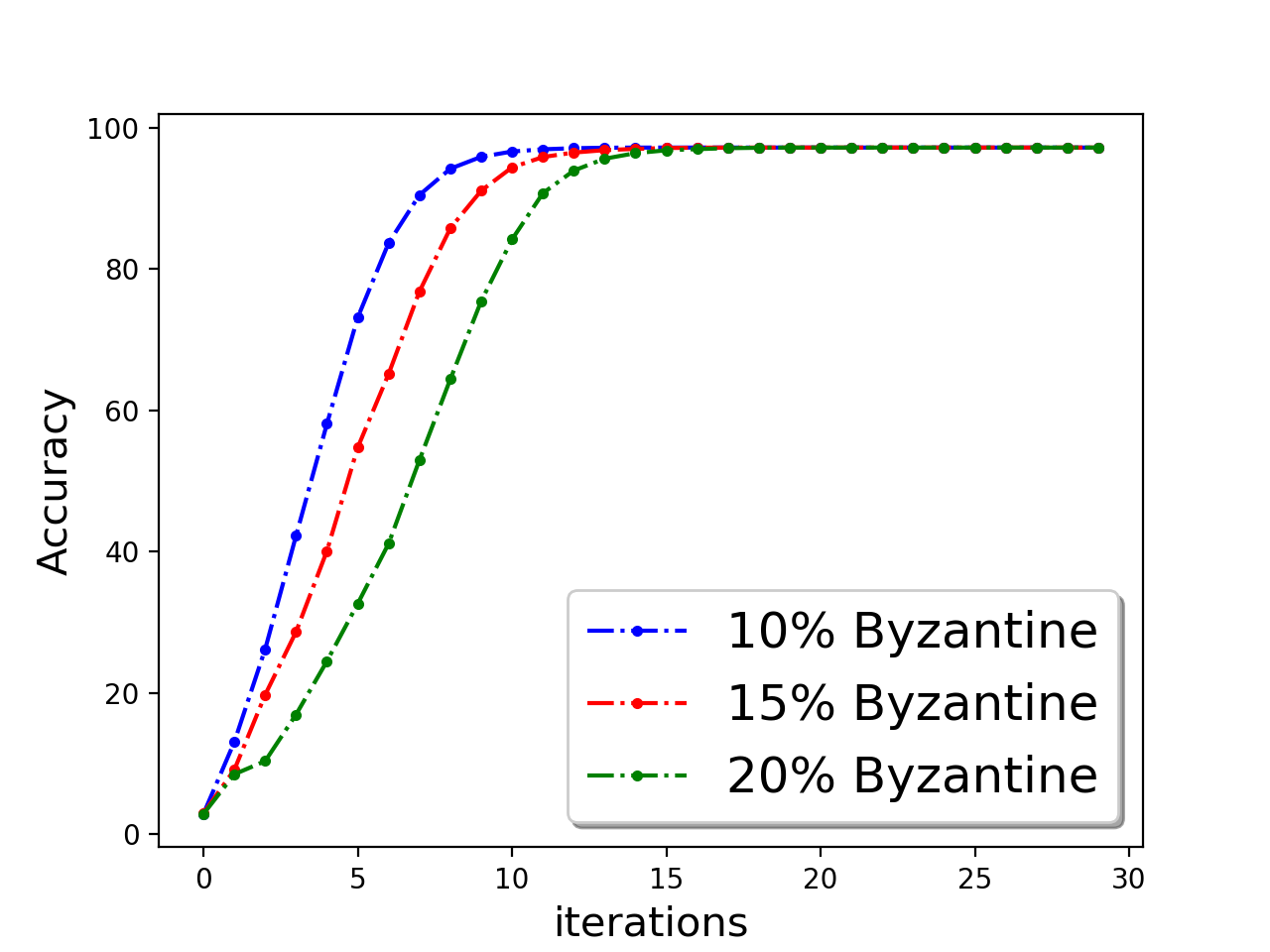} }}%
    \subfloat[w5q `negative']{{\includegraphics[height = 3cm,width=3.5cm]{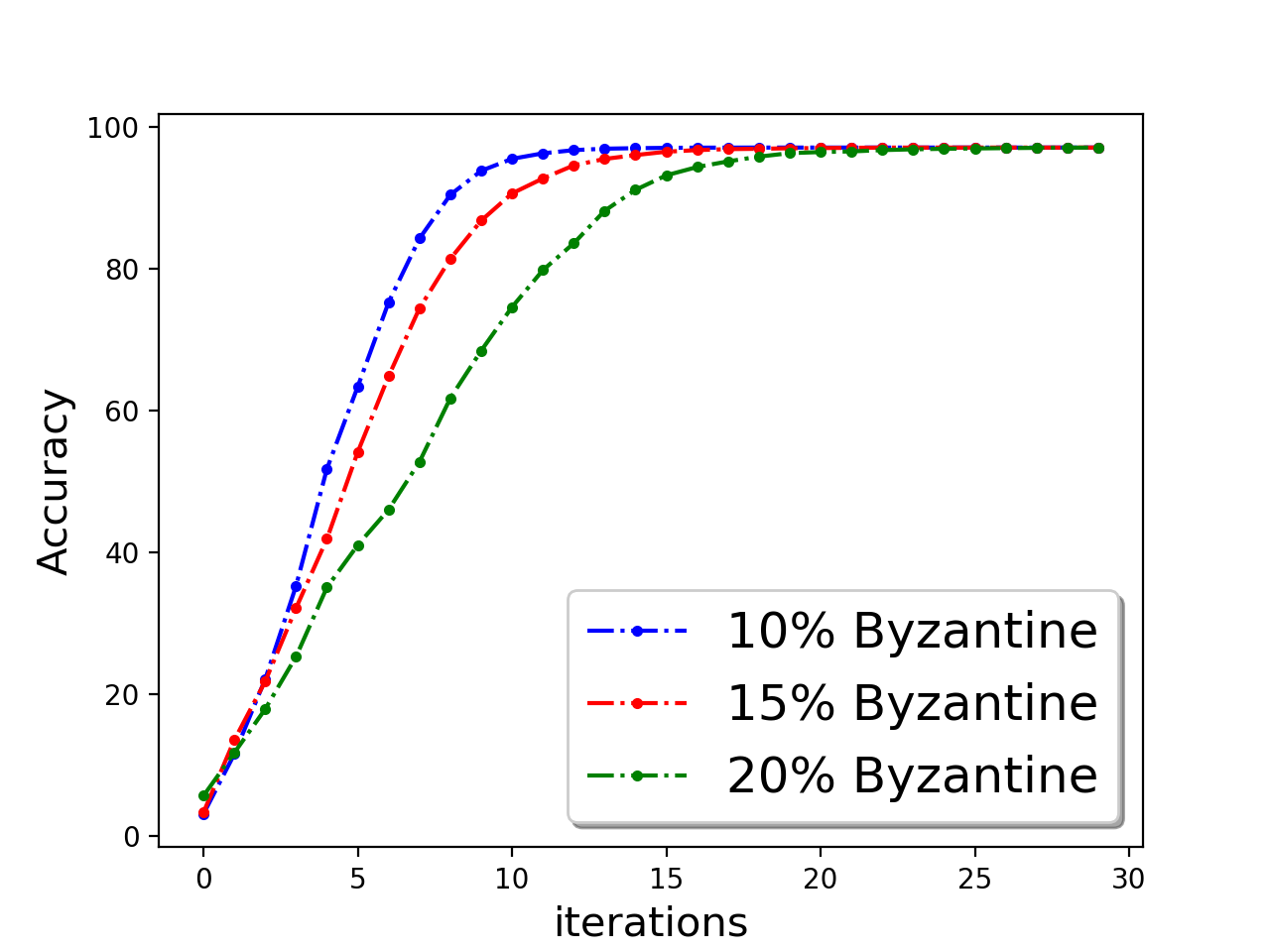} }}%
    \subfloat[a9a `flipped']{{\includegraphics[height = 3cm,width=3.5cm]{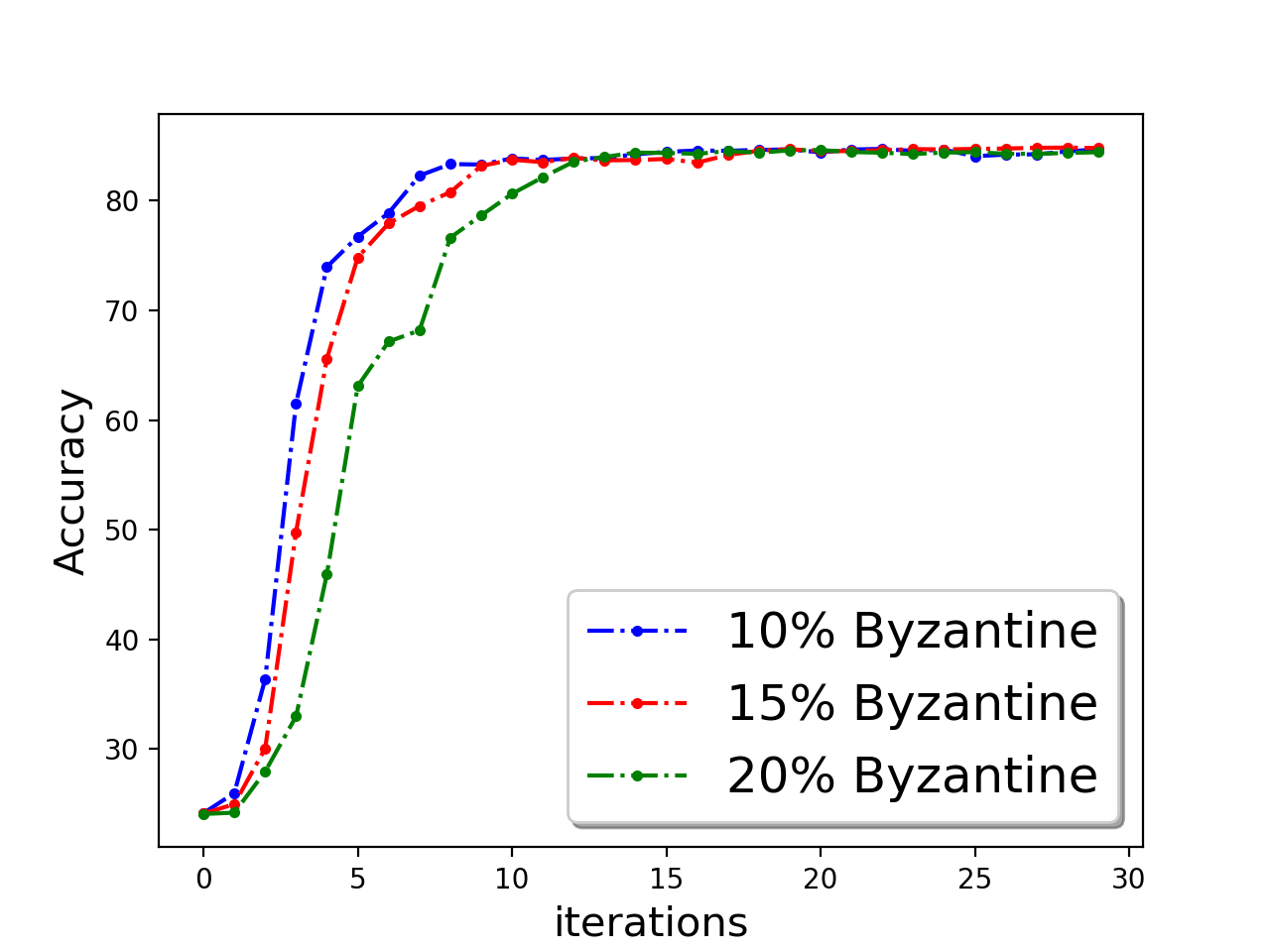} }}%
    \subfloat[a9a `negative']{{\includegraphics[height = 3cm,width=3.5cm]{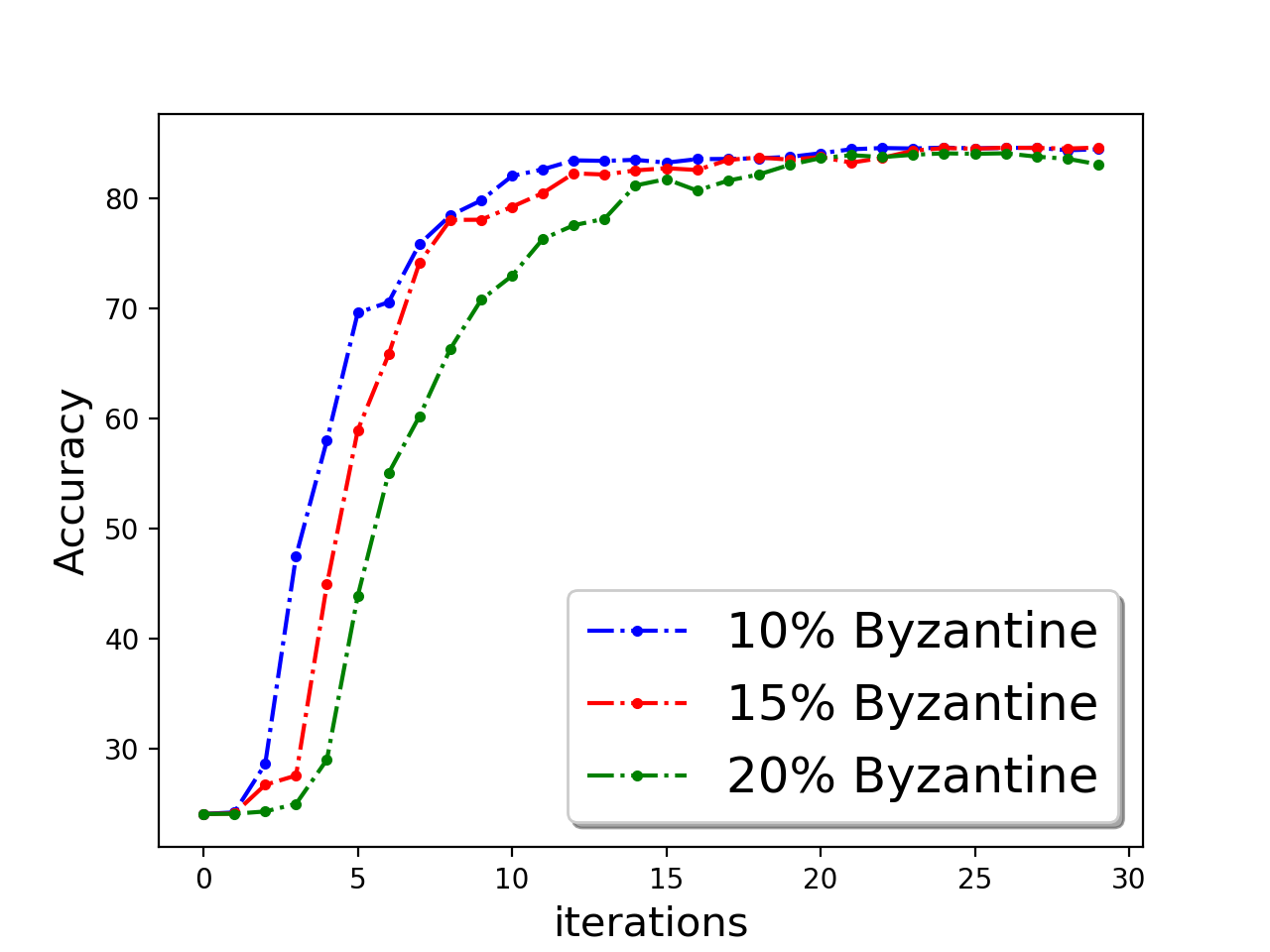} }}%
    \caption{(First row) Accuracy of \textsf{COMRADE}  with $10\%,15\%,20\%$ Byzantine workers with  `negative update' attack for (a).  w5a (b). a9a (c). covtype (d). Epsilon. (Second row) \textsf{COMRADE} accuracy  with $10\%,15\%,20\%$ Byzantine  workers with  `flipped label' attack for (e) w5a (f)   a9a (g) covtype (h) Epsilon.  (Third row) Accuracy of \textsf{COMRADE} with $\rho$-approximate compressor (Section~\ref{sec:compress}) with $10\%,15\%,20\%$ Byzantine  workers; (i) `flipped label' attack for w5a  (j) `negative update' attack for w5a. (k) `flipped label' attack for a9a . (l) `negative update' attack for a9a dataset.}%
    \label{fig:byz}%
\end{figure}

In Figure~\ref{fig:comparision}(first row) we compare \textsf{COMRADE} in non-Byzantine setup ($\alpha =\beta=0$) with the state-of the art algorithm GIANT~\cite{giant}. It is evident from the plot that despite the fact that \textsf{COMRADE} requires less communication, the algorithm is able to achieve similar accuracy.  Also, we show the ineffectiveness of GIANT in the presence of Byzantine attacks. In Figure ~\ref{fig:byz}((e),(f)) we show the accuracy for flipped label and negative update attacks. These plots are an indicator of the requirement of robustness in the learning algorithm. So we device `Robust GIANT', which is GIANT algorithm with added `norm based thresholding' for robustness. In particular, we trim the worker machines based on the local gradient norm in the first round of communication of GIANT. Subsequently, in the second round of communication, the non-trimmed worker machines send the updates (product of local Hessian inverse and the local gradient) to the center machine. We compare \textsf{COMRADE} with `Robust GIANT' in Figure~\ref{fig:comparision}((g),(h))  with $10\%$ Byzantine worker machines for `a9a' dataset. It is evident plot that \textsf{COMRADE} performs better than the `Robust GIANT'.

%

Next we show the accuracy of \textsf{COMRADE} with different numbers of  Byzantine worker machines.  Here we choose $c=0.9$. We show the accuracy for 'negetive update ' attack in Figure~\ref{fig:byz}(first row)  and 'flipped label' attack in Figure~\ref{fig:byz} (second row). Furthermore, we show that \textsf{COMRADE} works even when $\rho$-approximate compressor is applied to the updates. In Figure ~\ref{fig:byz}(Third row) we plot the tranning accuracies. For compression we apply the scheme known as QSGD \cite{qsgd}. Further experiments can be found in the supplementary material.

\section{Conclusion and Future Work}
In this paper, we address the issue of communication efficiency and Byzantine robustness via second order optimization and norm based thresholding respectively for strongly convex loss. Extending our setting to handle weakly convex and non-convex loss is of immediate interest. We would also like to exploit local averaging with second order optimization. Moreover, an import aspect, privacy, is not addressed in this work. We keep this as our future research direction.

\bibliographystyle{abbrv}
\bibliography{secondref}
%

\begin{center}
{\Large \textbf{Appendix}}
\end{center}

\section{Appendix A: Analysis of Section~\ref{sec:one}}
\textbf{Matrix Sketching} \\ Here we briefly discuss the matrix sketching that is broadly used in the context of \emph{randomized linear algebra}. For any matrix $\bA \in \mathbb{R}^{n \times d}$ the sketched  matrix  $\bZ \in \mathbb{R}^{s\times d}$ is defined as $\bS^T\bA$ where $\bS \in \mathbb{R}^{n\times s}$ is the sketching matrix (typically $s<n$). Based on the scope and basis of the application,  the sketched matrix is constructed by taking linear combination of the rows of matrix  which is known as \emph{random projection} or by sampling and scaling  a subset of the rows of the matrix which is known as \emph{random sampling}. The sketching is done to get a smaller representation of the original matrix  to reduce  computational cost.
  
Here we consider a uniform row sampling scheme. The matrix $\bZ$ is formed by sampling and scaling  rows of the matrix $\bA$.  Each row of the matrix $\bA $ is sampled with probability $p=\frac{1}{n}$ and scaled by multiplying with $\frac{1}{\sqrt{sp}}$ .
\begin{align*}
\mathbb{P}\left( \z_i= \frac{\ma_j}{\sqrt{sp}} \right)=p,
\end{align*}
 where $\z_i$ is the $i$-th  row  matrix $\bZ$  and  $\ma_j$ is the $j$ th row of the matrix $\bA$. Consequently the sketching matrix $\bS$ has one non-zero entry in each column. 

 We define the matrix $\bA_t^{\top}= [\ma_1^{\top},\ldots, \ma_n^{\top}]\in \mathbb{R}^{d\times n}$ where $\ma_j = \sqrt{\ell''_j(\w^{\top}\x_j)}\,\x_j$. So the exact Hessian in equation~\eqref{globalgradhess} is  $\bH_t= \frac{1}{n}\bA_t^{\top}\bA_t + \lambda \mathbf{I}$. Assume that $S_i$ is the set of features that are held by the $i$th worker machine. So the local Hessian is 
\begin{align*}
    \bH_{i ,t}= \frac{1}{s}\sum_{j\in S_i}\ell''_j(\w^{\top}\x_j)\x_j\x_j^{\top} +\lambda \mathbf{I}= \frac{1}{s}\bA_{i,t}^{\top}\bA_{i,t}+ \lambda \mathbf{I},
\end{align*}
where $\bA_{i,t}\in \mathbb{R}^{s\times d}$ and the row of the matrix $\bA_{i,t} $ is indexed by $S_i$.
Also we define $\bB_t= [\mb_1,\ldots ,\mb_n] \in \mathbb{R}^{d\times n}$ where $\mb_i= \ell'_i(\w^{\top}\x_i)\x_i$. So the exact gradient in equation~\eqref{globalgradhess} is $\g_t= \frac{1}{n}\bB_t\mathbf{1}+ \lambda\w_t$ and the local gradient is
\begin{align*}
\g_{i,t}= \frac{1}{s}\sum_{i \in S_i}\ell_j'(\w_t^{\top}\x_i)\x_i + \lambda\w_t =\frac{1}{s}\bB_{i,t}\mathbf{1} + \lambda\w_t,
\end{align*}
where $\bB_{i,t}$ is the matrix with column indexed by $S_i$. If $\{\bS_i\}_{i=1}^m$ are the sketching matrices then the local Hessian and gradient can be expressed as 
\begin{align}
\bH_{i ,t} = \bA_t^{\top}\bS_i\bS_i^{\top}\bA_t^{\top} + \lambda\mathbf{I} && \g_{i,t}= \frac{1}{n}\bB\bS_i\bS_i^{\top}\mathbf{1} + \lambda\w.
\end{align}
With the help of sketching idea later we show that the  local hessian and gradient are close to the exact hessian and gradient.

\textbf{The Quadratic function }  For the purpose of analysis we define an auxiliary  quadratic function 
\begin{align}
    \phi(\p)= \frac{1}{2}\p^{\top}\bH_t\p -\g_t^{\top}\p= \frac{1}{2}\p^{\top}(\bA_t^{\top}\bA_t+ \lambda \mathbf{I})\p -\g_t^{\top}\p. \label{quad}
\end{align}
The optimal solution to the above function is 
\begin{align*}
    \p^*= \arg\min \phi(\p) = \bH_t^{-1}\g_t= (\bA_t^{\top}\bA_t+ \lambda \mathbf{I})^{-1}\g_t,
\end{align*}
which is also the optimal  direction  of the global Newton update.  In this work we consider the local and  global (approximate ) Newton direction to be 
\begin{align*}
\hat{\p}_{i,t}= (\bA^{\top}\bS_i\bS_i^{\top}\bA+ \lambda \mathbf{I})^{-1}\g_{i,t}, \quad  \hat{\p}_t= \frac{1}{m}\sum_{i=1}^m \hat{\p}_{i,t}.
\end{align*}
respectively. And it can be easily verified that each local update $\hat{\p}_{i,t}$ is optimal solution to the following quadratic function
\begin{align}
\hat{\phi}_{i,t}(p)= \frac{1}{2}\p^{\top}(\bA^{\top}\bS_i\bS_i^{\top}\bA+ \lambda \mathbf{I})\p -\g_i^{\top}\p.
\end{align} 
In our convergence analysis we show that value of the quadratic function in \eqref{quad} with value $ \hat{\p}_t$ is close to the optimal value.

\textbf{Singular Value Decomposition (SVD)} For any matrix $\bA \in \mathbb{R}^{n \times d}$ with rank $r$, the singular value decomposition is defined as  $\bA = \bU\mathbf{\Sigma}\bV^{\top} $ where $\bU, \bV$ are  $n\times r$ and $d\times r$ column orthogonal matrices respectively and $\mathbf{\Sigma}$ is a $r\times r$ diagonal matrix with diagonal entries $\{\sigma_1,\ldots \sigma_r\}$. If $\bA$ is a symmetric positive semi-definite matrix then $\bU=\bV$.

\subsection{Analysis}

\begin{lemma}[McDiarmid's Inequality]\label{lem:mcd}
Let $X= X_1,\ldots ,X_m$ be $m$ independent random variables taking values from some set $A$, and assume that $f: A^m \rightarrow \mathbb{R}$ satisfies the following condition (bounded differences ):
\begin{align*} 
\sup_{x_1,\ldots ,x_m,\hat{x}_i} \left|  f(x_i ,\ldots,x_i,\ldots,x_m)- f(x_i ,\ldots,\hat{x}_i,\ldots,x_m) \right| \leq c_i,
\end{align*}
for all $ i \in  \{1,\ldots ,m \}$. Then for any $\epsilon>0 $ we have 
 
 \begin{align*}
 P\left[ f(X_1,\ldots ,X_m)- \mathbb{E}[f(X_1,\ldots ,X_m)] \geq \epsilon \right] \leq \exp \left( - \frac{2\epsilon^2}{\sum_{i=1}^mc_i^2} \right).
 \end{align*}
\end{lemma}
The property described  in the following  Lemma~\ref{lem:Hessketch} is a very useful result for  uniform row sampling sketching matrix. 

\begin{lemma}[Lemma 8 \cite{giant}]\label{lem:Hessketch}
Let $\eta,\delta \in (0,1)$ be  a fixed parameter and $r= \text{rank}(\bA_t)$ and $\bU \in \mathbb{R}^{n \times r}$ be the orthonormal bases of the matrix $\bA_t$. Let $\{\bS_i\}_{i=1}^m$  be sketching matrices and $\bS= \frac{1}{\sqrt{m}}[\bS_1,\ldots \bS_m] \in \mathbb{R}^{n\times ms}$. With probability $1-\delta$ the following holds
\begin{align*}
\left\|\bU^{\top}\bS_i\bS_i^{\top}\bU - \mathbf{I} \right\|_2 \leq \eta \quad \forall i \in [m]\quad   \text{ and } \quad  \left\|\bU^{\top}\bS\bS^{\top}\bU - \mathbf{I} \right\|_2 \leq \frac{\eta}{\sqrt{m}}.
\end{align*} 
\end{lemma}

\begin{lemma}\label{lem:gradsketch}
Let $\bS \in \mathbb{R}^{n \times s}$  be any uniform sampling sketching matrix, then for any matrix  $\bB= [\mb_1,\ldots ,\mb_n] \in \mathbb{R}^{d\times n}$ with probability $1-\delta$ for any $\delta>0$ we have, 
\begin{align*}
   \|\frac1n \bB\bS\bS^{\top}\mathbf{1}- \frac1n \bB\mathbf{1}\| \leq  (1+ \sqrt{2\ln (\frac{1}{\delta})})\sqrt{\frac{1}{s}}\max_i \|\mb_i\|,
\end{align*}
where $\mathbf{1}$ is  all ones vector. 
\end{lemma}

\begin{proof}
The vector $ \bB\mathbf{1}$ is the sum of column of the matrix $\bB$ and  $\bB\bS\bS^{\top}\mathbf{1}$ is the sum of uniformly sampled and scaled column of the matrix $\bB$ where the scaling factor is $\frac{1}{\sqrt{sp}}$ with $p=\frac{1}{n}$. If $(i_1,\ldots ,i_s)$ is the set of sampled indices then $\bB\bS\bS^{\top}\mathbf{1}= \sum_{k\in (i_1,\ldots ,i_s) }\frac{1}{sp}\mb_k$.
 
Define the function $f(i_1,\ldots ,i_s)=\|\frac1n \bB\bS\bS^{\top}\mathbf{1}- \frac1n\bB\mathbf{1}\|$. Now consider a sampled set $(i_1,\ldots,i_{j'},\ldots ,i_s)$ with only one item (column) replaced then the bounded difference is 
\begin{align*}
    \Delta&= |f(i_1,\ldots,i_j,\ldots ,i_s)-f(i_1,\ldots,i_{j'},\ldots ,i_s)|\\
    & =|\frac1n \|\frac{1}{sp}\mb_{i_j'}-\frac{1}{sp}\mb_{i_j} \| | \leq \frac{2}{s}\max_i\|\mb_{i}\|.
\end{align*}
Now we have the expectation 
\begin{align*}
    \mathbb{E}[\| \frac1n \bB\bS\bS^{\top}\mathbf{1}- \frac1n\bB\mathbf{1}\|^2] &\leq \frac{n}{sn^2}\sum_{i=1}^n \|\mb_i\|^2= \frac{1}{s} \max_i \|\mb_i\|^2 \\
 \Rightarrow   \mathbb{E}[\|\frac1n \bB\bS\bS^{\top}\mathbf{1}- \frac1n \bB\mathbf{1}\|] & \leq \sqrt{\frac{1}{s}}\max_i \|\mb_i\|.
\end{align*}
Using McDiarmid inequality (Lemma~\ref{lem:mcd}) we have 
\begin{align*}
    P[\left\|\frac1n \bB\bS\bS^{\top}\mathbf{1}- \frac1n \bB\mathbf{1}\|\geq \sqrt{\frac{1}{s}}\max_i \|\mb_i\| + t \right] \leq \exp\left(- \frac{2t^2}{s\Delta^2} \right).
\end{align*}
Equating the probability with $\delta$ we have 
\begin{align*}
     & \exp(- \frac{2t^2}{s\Delta^2})  = \delta \\
\Rightarrow &    t =\Delta \sqrt{\frac{s}{2}\ln (\frac{1}{\delta})} = \max_i \|\mb_i\|\sqrt{\frac{2}{s}\ln (\frac{1}{\delta})}.
\end{align*}
Finally we have  with probability $1-\delta$
\begin{align*}
   \|\frac1n \bB\bS\bS^{\top}\mathbf{1}- \frac1n \bB\mathbf{1}\| \leq  (1+ \sqrt{2\ln (\frac{1}{\delta})})\sqrt{\frac{1}{s}}\max_i \|\mb_i\|.
\end{align*}
\end{proof}
\begin{remark}
For $m$ sketching matrix  $\{\bS_i\}_{i=1}^m$, the bound in the Lemma~\ref{lem:gradsketch} is
\begin{align*}
   \|\frac1n \bB\bS_i\bS_i^{\top}\mathbf{1}- \frac1n \bB\mathbf{1}\| \leq  (1+ \sqrt{2\ln (\frac{m}{\delta})})\sqrt{\frac{1}{s}}\max_i \|\mb_i\|,
\end{align*}
with probability $1-\delta$  for any $\delta>0  $ for all $i \in \{1,2,\ldots ,m\}$. In the case that each worker machine holds data based on the uniform sketching matrix the local gradient is close to the exact gradient.  Thus the local second order update acts as a good approximate to the exact Netwon update. 
 \end{remark}

Now we consider the update rule of GIANT \cite{giant} where the update is done in two rounds in each iteration. In the first round each worker machine computes and send the local gradient and the center machine computes the exact gradient $\g_t$ in iteration $t$. Next the center machine broadcasts the exact gradient and each worker machine computes the local Hessian and send $ \tilde{\p}_{i,t}=(\bH_{i,t})^{-1}\g_t$ to the center machine and  the center machine computes the approximate Newton direction $\tilde{\p}_t= \frac{1}{m}\sum_{i=1}^m \tilde{\p}_{i,t}$. Now based on this we restate the following lemma (Lemma 6 \cite{giant}).

\begin{lemma}\label{lem:giant}
Let $\{\bS_i\}_{i=1}^m \in \mathbb{R}^{n \times s}$  be sketching matrices based on  Lemma~\ref{lem:Hessketch}. Let $\phi_t$ be defined in \eqref{quad} and $\tilde{\p}_t$ be the update. It holds that 
\begin{align*}
\min_{\p}\phi_t(\p) \leq \phi_t(\tilde{\p}_t) \leq (1 -\zeta^2) \min_{\p}\phi_t(\p),
\end{align*}
where $\zeta= \nu(\frac{\eta}{\sqrt{m}}+ \frac{\eta^2}{1-\eta})$ and $\nu= \frac{\sigma_{max}(\bA^{\top}\bA)}{\sigma_{max}(\bA^{\top}\bA)+n\lambda} \leq 1$.
\end{lemma}

Now we prove similar guarantee for the update according to \textsf{COMRADE} in Algorithm~\ref{alg:main_algo}.
\begin{lemma}\label{lem:onernd}
Let $\{\bS_i\}_{i=1}^m \in \mathbb{R}^{n \times s}$  be sketching matrices based on  Lemma~\ref{lem:Hessketch}. Let $\phi_t$ be defined in \eqref{quad} and $\hat{\p}_t$ be defined in  Algorithm~\ref{alg:main_algo}($\beta=0$)
\begin{align*}
\min_{\p}\phi_t(\p) \leq \phi_t(\hat{\p}_t) \leq \epsilon^2+ (1 -\zeta^2) \min_{\p}\phi_t(\p),
\end{align*}
where $\epsilon =\frac{1}{1-\eta}\frac{1}{\sqrt{\sigma_{min}(\bH_t)}} (1+ \sqrt{2\ln (\frac{m}{\delta})})\sqrt{\frac{1}{s}}\max_i \|\mb_i\| $ and $\zeta= \nu(\frac{\eta}{\sqrt{m}}+ \frac{\eta^2}{1-\eta})$ and $\nu= \frac{\sigma_{max}(\bA^{\top}\bA)}{\sigma_{max}(\bA^{\top}\bA)+n\lambda}$.
\end{lemma}

\begin{proof}
First consider the  quadratic function \eqref{quad}
\begin{align}
\phi_t(\hat{\p}_t)- \phi_t(\p^*)& =\frac{1}{2} \|\bH_t^{\frac12}(\hat{\p}_t-\p^*)\|^2 \nonumber  \\
    &\leq  \underbrace{(\|\bH_t^{\frac12}(\hat{\p}_t-\Tilde{\p}_t)\|^2}_{Term 1}+\underbrace{ \|\bH_t^{\frac12}(\Tilde{\p}_t-\p^*)\|)^2}_{Term 2}, \label{twoterm}
\end{align}
where $\Tilde{\p}_t= \frac{1}{m}\sum_{i=1}^m (\bH_{i,t})^{-1}\g_t$. First we bound the  Term 2 of \eqref{twoterm} using the quadratic function and Lemma~\ref{lem:giant} 
\begin{align}
 \frac{1}{2}\left\|\bH_t^{\frac12}(\Tilde{\p}_t-\p^*)\right\|)^2 
& \leq \zeta^2\left\|\bH_t^{\frac12}\p^* \right\|^2 \quad \text{ (Using Lemma~\ref{lem:giant}  )} \nonumber \\
& = - \zeta^2 \phi_t(\p^*). \label{term2}
\end{align}
The step in equation~\eqref{term2} is from the definition of the function $\phi_t$ and $\p^*$. It can be shown that 
\begin{align*}
\phi_t(\p^*)= -\left\|\bH_t^{\frac12}\p^* \right\|^2.
\end{align*}
Now we bound the Term 1 in \eqref{twoterm}. By Lemma~\ref{lem:Hessketch}, we have $(1-\eta)\bA_t^{\top}\bA_t \preceq \bA_t^{\top}\bS_i\bS_i^{\top}\bA_t \preceq (1+ \eta)\bA_t^{\top}\bA_t$. Following we have $ (1-\eta)\bH_t \preceq \bH_{i,t} \preceq (1+ \eta)\bH_t$. Thus there exists matrix $\mathbf{\xi}_i$ satisfying
\begin{align*}
\bH_t^{\frac12}\bH_{i,t}^{-1} \bH_t^{\frac12}= \mathbf{I}+ \mathbf{\xi}_i \quad \text{and } \quad - \frac{\eta}{1+\eta}\preceq \mathbf{\xi}_i \preceq \frac{\eta}{1-\eta},
\end{align*}
So we have,
\begin{align}
\left\| \bH_t^{\frac12}\bH_{i,t}^{-1} \bH_t^{\frac12} \right\| \leq 1+ \frac{\eta}{1-\eta}= \frac{1}{1-\eta}. \label{htl2}
\end{align}
Now we have 
\begin{align}
 \left \|\bH_t^{\frac12}(\hat{\p}_t-\Tilde{\p}_t)\right\|
& =\left\|\bH_t^{\frac12} \frac{1}{m}\sum_{i=1}^m(\hat{\p}_{i,t}-\Tilde{\p}_{i,t})\right\| \nonumber \\
& \leq \frac{1}{m}\sum_{i=1}^m \left\|\bH_t^{\frac12}(\hat{\p}_{i,t}-\Tilde{\p}_{i,t})\right\| \nonumber\\
& =  \frac{1}{m}\sum_{i=1}^m \left\|\bH_t^{\frac12}\bH_{i,t}^{-1}(\g_{i,t}-\g_{t})\right\| \nonumber \\
&= \frac{1}{m}\sum_{i=1}^m \left\|\bH_t^{\frac12}\bH_{i,t}^{-1}\bH_t^{\frac12}\bH_t^{-\frac12}(\g_{i,t}-\g_{t})\right\| \nonumber \\
& \leq  \frac{1}{m}\sum_{i=1}^m \left\|\bH_t^{\frac12}\bH_{i,t}^{-1}\bH_t^{\frac12}\right\| \left\|\bH_t^{-\frac12}(\g_{i,t}-\g_{t})\right\| \nonumber \\
& \leq \frac{1}{1-\eta} \frac{1}{m}\sum_{i=1}^m \left\|\bH_t^{-\frac12}(\g_{i,t}-\g_{t})\right\|  \quad \text{  ( Using \eqref{htl2})}\nonumber \\
& \leq  \frac{1}{1-\eta}\frac{1}{\sqrt{\sigma_{min}(\bH_t)}} \frac{1}{m}\sum_{i=1}^m \left\|(\g_{i,t}-\g_{t})\right\|. \label{term1}
\end{align}

Now we bound $ \left\|(\g_{i,t}-\g_{t})\right\| $ using Lemma~\ref{lem:gradsketch},
\begin{align*}
\left\|(\g_{i,t}-\g_{t})\right\|  =  \|\frac1n \bB\bS\bS^{\top}\mathbf{1}- \frac1n \bB\mathbf{1}\| \leq  (1+ \sqrt{2\ln (\frac{m}{\delta})})\sqrt{\frac{1}{s}}\max_i \|\mb_i\|.
\end{align*}
Plugging it into equation~\eqref{term1} we get,
\begin{align}
\left \|\bH_t^{\frac12}(\hat{\p}_t-\Tilde{\p}_t)\right\| & \leq \frac{1}{1-\eta}\frac{1}{\sqrt{\sigma_{min}(\bH_t)}} \frac{1}{m}\sum_{i=1}^m \left\|(\g_{i,t}-\g_{t})\right\| \nonumber \\
& \leq \frac{1}{1-\eta}\frac{1}{\sqrt{\sigma_{min}(\bH_t)}} (1+ \sqrt{2\ln (\frac{m}{\delta})})\sqrt{\frac{1}{s}}\max_i \|\mb_i\|. \label{term1_1}
\end{align}

Now collecting the terms of \eqref{term1_1} and \eqref{term2}  and plugging them into \eqref{twoterm} we have 
\begin{align*}
\phi_t(\hat{\p}_t)- \phi_t(\p^*) \leq  \epsilon^2 - \zeta^2 \phi_t(\p^*) \\
\Rightarrow  \phi_t(\hat{\p}_t) \leq \epsilon^2 + (1- \zeta^2) \phi_t(\p^*),
\end{align*}
where  $\epsilon$ is as defined in ~\eqref{eps}.

\end{proof}

\begin{lemma}\label{lem:delta}
Let $\zeta \in (0,1),\epsilon$ be any fixed parameter. And $\hat{p}_t$ satisfies $\phi_t(\hat{\p}_t) \leq \epsilon^2+ (1 -\zeta^2) \min_{\p}\phi_t(\p)$. Under the Assumption~\ref{asm:hess}(Hessian $L$-Lipschitz) and 
$\mathbf{\Delta}_t =\w_t-\w^*$ satisfies
\begin{align*}
\mathbf{\Delta}^{\top}_{t+1}\bH_t\mathbf{\Delta}_{t+1}& \leq L\| \mathbf{\Delta}_{t+1}\|\|\mathbf{\Delta}_t\|^2 +  \frac{\zeta^2}{1-\zeta^2}\mathbf{\Delta}_t^{\top}\bH_t\mathbf{\Delta}_t + 2\epsilon^2.
\end{align*}
\end{lemma}

\begin{proof}

We have $\w_{t+1}=\w_t- \hat{\p}_t , \mathbf{\Delta}_t= \w_t-\w^* \text{   and  } \mathbf{\Delta}_{t+1}= \w_{t+1}-\w^*$. Also $\hat{\p}_t= \w_t-\w_{t+1}=\mathbf{\Delta}_t -\mathbf{\Delta}_{t+1}$. From the definition of $\phi$ we have,
\begin{align*}
    \phi_t(\hat{\p}_t)& = \frac{1}{2}(\mathbf{\Delta}_t-\mathbf{\Delta}_{t+1})^{\top}\bH_t(\mathbf{\Delta}_t-\mathbf{\Delta}_{t+1})- \left(\mathbf{\Delta}_t -\mathbf{\Delta}_{t+1})\right)\g_t,\\
    (1-\zeta^2)\phi_t(\frac{1}{(1-\zeta^2)}\mathbf{\Delta}_t) &=\frac{1}{2(1-\zeta^2)}\mathbf{\Delta}_t^{\top}\bH_t\mathbf{\Delta}_t-\mathbf{\Delta}_t^{\top}\g_t.
\end{align*}
From the above two equation we have 
\begin{align*}
   & \phi_t(\hat{\p}_t) -(1-\zeta^2)\phi_t(\frac{1}{(1-\zeta^2)}\mathbf{\Delta}_t) \\
   & = \frac{1}{2}\mathbf{\Delta}^{\top}_{t+1}\bH_t\mathbf{\Delta}_{t+1} - \frac{1}{2} \mathbf{\Delta}_t^{\top}\bH_t\mathbf{\Delta}_{t+1}+ \frac{1}{2} \mathbf{\Delta}^{\top}_{t+1}\g_t-\frac{\zeta^2}{2(1-\zeta^2)}\mathbf{\Delta}_t^{\top}\bH_t\mathbf{\Delta}_t .
\end{align*} 
From Lemma~\ref{lem:onernd} the following holds 
\begin{align*}
\phi_t(\hat{\p}_t)&  \leq \epsilon^2+ (1 -\zeta^2) \min_{\p}\phi_t(\p) \\
& \leq  \epsilon^2+ (1 -\zeta^2)\phi_t(\frac{1}{(1-\zeta^2)}\mathbf{\Delta}_t).
\end{align*}
So we have 
\begin{align}
 \frac{1}{2}\mathbf{\Delta}^{\top}_{t+1}\bH_t\mathbf{\Delta}_{t+1} -\mathbf{\Delta}_t^{\top}\bH_t\mathbf{\Delta}_{t+1}+\mathbf{\Delta}^{\top}_{t+1}\g_t-\frac{\zeta^2}{2(1-\zeta^2)}\mathbf{\Delta}_t^{\top}\bH_t\mathbf{\Delta}_t  \leq \epsilon^2. \label{delta1}
\end{align}
 Consider $\g_t = \g(\w_t)$
\begin{align*}
\g(\w_t) & =\g(\w^*) + \left( \int_{0}^1 \nabla^2 f(\w^* + z(\w_t-\w^*) )dz \right)(\w_t-\w^*) \\
& =   \left( \int_{0}^1 \nabla^2 f(\w^* + z(\w_t-\w^*) )dz \right)\mathbf{\Delta}_t \quad \text{(as $\g(\w^*)=0$)}.
\end{align*}
Now we bound the following 
\begin{align*}
\left\| \bH_t\mathbf{\Delta}_t -\g(\w_t) \right\| & \leq \left\| \mathbf{\Delta}_t \right\|  \left\|   \int_{0}^1  [\nabla^2 f(\w_t) -\nabla^2 f(\w^* + z(\w_t-\w^*) )]dz  \right\| \\
&\leq \left\| \mathbf{\Delta}_t \right\|  \int_{0}^1  \left\|   [\nabla^2 f(\w_t) -\nabla^2 f(\w^* + z(\w_t-\w^*) )]  \right\| dz \quad \text{(By Jensen's Inequality)} \\
& \leq  \left\| \mathbf{\Delta}_t \right\|  \int_{0}^1 (1-z)L \left\|\w_t-\w^* \right\|dz \quad \text{(by $L$-Lipschitz assumption)} \\
& = \frac{L}{2}\left\| \mathbf{\Delta}_t \right\|^2.
\end{align*} 
Plugging it into \eqref{delta1} we have 
\begin{align*}
\mathbf{\Delta}^{\top}_{t+1}\bH_t\mathbf{\Delta}_{t+1} & \leq 2 \mathbf{\Delta}^{\top}_{t+1} \left( \bH_t\mathbf{\Delta}_t -\g_t \right) + \frac{\zeta^2}{(1-\zeta^2)}\mathbf{\Delta}_t^{\top}\bH_t\mathbf{\Delta}_t +  2\epsilon^2 \\
& \leq 2 \left\|\mathbf{\Delta}_{t+1}\right\| \left\| \bH_t\mathbf{\Delta}_t -\g_t \right\| + \frac{\zeta^2}{(1-\zeta^2)}\mathbf{\Delta}_t^{\top}\bH_t\mathbf{\Delta}_t +  2\epsilon^2 \\
& \leq  L\left\|\mathbf{\Delta}_{t+1}\right\| \left\| \mathbf{\Delta}_t \right\|^2 + \frac{\zeta^2}{(1-\zeta^2)}\mathbf{\Delta}_t^{\top}\bH_t\mathbf{\Delta}_t +  2\epsilon^2.
\end{align*}
\end{proof}

\textbf{Proof of Theorem~\ref{thm:smooth}}
\begin{proof}
From the Lemma~\ref{lem:delta}  with probability $1-\delta$
\begin{align*}
\mathbf{\Delta}^{\top}_{t+1}\bH_t\mathbf{\Delta}_{t+1}  & \leq L \left\|\mathbf{\Delta}_{t+1}\right\| \left\| \mathbf{\Delta}_t \right\|^2 + \frac{\zeta^2}{(1-\zeta^2)}\mathbf{\Delta}_t^{\top}\bH_t\mathbf{\Delta}_t +  2\epsilon^2 \\
& \leq L\| \mathbf{\Delta}_{t+1}\|\|\mathbf{\Delta}_t\|^2 + (\frac{\zeta^2}{1-\zeta^2}\sigma_{max}(\bH_t)) \|\mathbf{\Delta}_{t}\|^2 + 2\epsilon^2.
\end{align*}
So we have,
\begin{align*}
    \| \mathbf{\Delta}_{t+1}\| \leq \max \{ \sqrt{\frac{\sigma_{max}(\bH_t)}{\sigma_{min}(\bH_t)}(\frac{\zeta^2}{1-\zeta^2})}\| \mathbf{\Delta}_{t}\|, \frac{L}{\sigma_{min}(\bH_t) }\| \mathbf{\Delta}_{t}\|^2  \}+ \frac{ 2\epsilon}{\sqrt{\sigma_{min}(\bH_t)}}.
\end{align*}
\end{proof}

\section{Appendix B: Analysis of Section~\ref{sec:byz}}
In this section we provide the theoretical analysis of the Byzantine robust 
method explained in Section~\ref{sec:byz} and prove the statistical  guarantee. In any iteration  $t$  the following holds
\begin{align*}
    |\cU_t| &  = |(\cU_t\cap\cM_t)|  + |(\cU_t\cap\cB_t)|  \\   
    |\cM_t| & = |(\cU_t\cap\cM_t)| + |(\cM_t\cap\cT_t)|.
\end{align*}
Combining both we have 
\begin{align*}
    |\cU_t| = |\cM_t| - |(\cM_t\cap\cT_t)|+|(\cU_t\cap\cB_t)|.
\end{align*}

\begin{lemma}\label{lem:byzone}
Let $\{\bS_i\}_{i=1}^m \in \mathbb{R}^{n \times s}$  be sketching matrices based on  Lemma~\ref{lem:Hessketch}. Let $\phi_t$ be defined in \eqref{quad} and $\hat{\p}_t$ be defined in  Algorithm~\ref{alg:main_algo}. It holds that 
\begin{align*}
\min_{\p}\phi_t(\p) \leq \phi_t(\hat{\p}_t)& \leq \epsilon_{byz}^2 + (1 - \zeta^2_{byz})\phi(\p^*),
\end{align*}
where $\epsilon_{byz} $ and $\zeta_{byz}$ is defined in ~\eqref{gbyz} and ~\eqref{alpbyz} respectively.

\end{lemma}
\begin{proof}
In the following analysis we omit the subscript '$t$'. From the definition of the quadratic function \eqref{quad} we know that    
\begin{align*}
    \phi(\hat{\p}) - \phi(\p^*) & = \frac{1}{2}\|\bH^{\frac12}(\hat{\p}-\p^*)\|^2.
\end{align*}
Now we consider 
\begin{align*}
  \frac{1}{2} \|\bH^{\frac12}(\hat{\p}-\p^*)\|^2 & = \frac{1}{2}\|\bH^{\frac12}(\frac{1}{|\cU|}\sum_{i \in \cU}\hat{\p}_i-\p^*)\|^2  \\
   &=\frac{1}{2}\|\bH^{\frac12}\frac{1}{|\cU|}(\sum_{i \in \cM}(\hat{\p}_i-\p^*) -\sum_{i \in (\cM\cap\cT)}(\hat{\p}_i-\p^*)+\sum_{i \in (\cU\cap\cB)}(\hat{\p}_i-\p^*) )\|^2 \\
   &\leq \underbrace{ \|\bH^{\frac12}\frac{1}{|\cU|}(\sum_{i \in \cM}(\hat{\p}_i-\p^*)\|^2}_{Term 1} + \underbrace{ 2\|\bH^{\frac12}\frac{1}{|\cU|}\sum_{i \in (\cM\cap\cT)}(\hat{\p}_i-\p^*)\|^2 }_{Term 2}+ \underbrace{2\|\bH^{\frac12}\frac{1}{|\cU|}\sum_{i \in (\cU\cap\cB)}(\hat{\p}_i-\p^*) )\|^2}_{Term 3}.
\end{align*}
Now we bound each term separately and use the result of the Lemma~\ref{lem:onernd} to bound each term.
\begin{align*}
    Term 1 &=  \|\bH^{\frac12}\frac{1}{|\cU|}(\sum_{i \in \cM}(\hat{\p}_i-\p^*)\|^2 \\& = (\frac{1-\alpha}{1-\beta})^2 \|\bH^{\frac12}\frac{1}{|\cM|}(\sum_{i \in \cM}(\hat{\p}_i-\p^*)\|^2 \\
   & \leq (\frac{1-\alpha}{1-\beta})^2 [\epsilon^2+ \zeta^2_{\cM}\|\bH^{\frac{1}{2}}\p^*\|^2],
\end{align*}
where $\zeta_{\cM}= \nu(\frac{\eta}{\sqrt{|\cM|}}+ \frac{\eta^2}{1-\eta})=\nu(\frac{\eta}{\sqrt{(1-\alpha)m}}+ \frac{\eta^2}{1-\eta})$.

Similarly  the Term 2 can be bonded as it is a bound on good machines
\begin{align*}
    Term 2 & = 2\|\bH^{\frac12}\frac{1}{|\cU|}\sum_{i \in (\cM\cap\cT)}(\hat{\p}_i-\p^*)\|^2 \\
    & = 2 (\frac{1-\alpha}{1-\beta})^2 \|\bH^{\frac12}\frac{1}{|\cM\cap\cT|}\sum_{i \in (\cM\cap\cT)}(\hat{\p}_i-\p^*)\|^2\\
    & \leq 2(\frac{1-\alpha}{1-\beta})^2 [\epsilon^2 + \zeta^2_{\cM\cap\cT}\|\bH^{\frac{1}{2}}\p^*\|^2],
\end{align*}
where  $\zeta_{\cM\cap\cT}= \nu(\frac{\eta}{\sqrt{|\cM\cap\cT|}}+ \frac{\eta^2}{1-\eta}) \leq  \nu(\frac{\eta}{\sqrt{(1-\beta)m}}+ \frac{\eta^2}{1-\eta})$.

For the Term 3 we know that $\beta > \alpha $ so all the untrimmed worker norm is bounded by a good machine as at least one good machine gets trimmed.
\begin{align*}
    Term 3 & = 2\|\bH^{\frac12}\frac{1}{|\cU|}\sum_{i \in (\cU\cap\cB)}(\hat{\p}_i-\p^*) )\|^2 \\
    & \leq 2\sigma_{max}(\bH)(\frac{|\cU\cap\cB|}{|\cU|})^2 \|\frac{1}{|\cU\cap\cB|}\sum_{i \in (\cU\cap\cB)}(\hat{\p}_i-\p^*) )\|^2 \\
    & \leq 2\sigma_{max}(\bH)(\frac{|\cU\cap\cB|}{|\cU|})^2 \frac{1}{|\cU\cap\cB|}\sum_{i \in (\cU\cap\cB)}\|(\hat{\p}_i-\p^*) )\|^2\\
   & \leq 4\sigma_{max}(\bH)(\frac{|\cU\cap\cB|}{|\cU|})^2 \frac{1}{|\cU\cap\cB|}\sum_{i \in (\cU\cap\cB)}(\|\hat{\p}_i\|^2+ \|\p^*\|^2) \\
   & \leq 4\sigma_{max}(\bH)(\frac{|\cU\cap\cB|}{|\cU|})^2 \max_{i \in \cM}(\|\hat{\p}_i\|^2+ \|\p^*\|^2)\\
   & \leq 4\sigma_{max}(\bH)(\frac{|\cU\cap\cB|}{|\cU|})^2 \max_{i \in \cM}(\|\hat{\p}_i- \p^*\|^2+ 2\|\p^*\|^2)\\
   & \leq 4\kappa (\frac{|\cU\cap\cB|}{|\cU|})^2 \max_{i \in \cM}(\|\bH^{\frac{1}{2}}(\hat{\p}_i- \p^*)\|^2+ 2\|\bH^{\frac{1}{2}}\p^*\|^2)\\
   &\leq   4\kappa (\frac{|\cU\cap\cB|}{|\cU|})^2(\epsilon^2 + (2+\zeta^2_{1})\|\bH^{\frac{1}{2}}\p^*\|^2)\\
   & \leq 4\kappa (\frac{\alpha}{1-\beta})^2(\epsilon^2 + (2+\zeta^2_{1})\|\bH^{\frac{1}{2}}\p^*\|^2),
\end{align*}
where  $\zeta_{1}= \nu(\eta+ \frac{\eta^2}{1-\eta}) = \frac{\nu}{1-\eta}$ and $\kappa= \frac{\sigma_{max}(\bH)}{\sigma_{min}(\bH)}$.

Combining all the bounds on Term1 , Term2 and Term3 we have 
\begin{align*}
  \frac{1}{2} \|\bH^{\frac12}(\hat{\p}-\p^*)\|^2 & \leq    \epsilon_{byz}^2+ \zeta^2_{byz}\|\bH^{\frac{1}{2}}\p^*\|^2,
\end{align*}
where 
\begin{align*}
    \epsilon_{byz}^2 &= \left(3 \left(\frac{1-\alpha}{1-\beta}\right)^2 +4\kappa \left(\frac{\alpha}{1-\beta}\right)^2 \right)\epsilon^2, \\
    \zeta^2_{byz} & =2 \left(\frac{1-\alpha}{1-\beta} \right)^2\zeta^2_{\cM\cap\cT} + \left(\frac{1-\alpha}{1-\beta} \right)^2\zeta^2_{\cM} + 4\kappa \left(\frac{\alpha}{1-\beta}\right)^2(2+\zeta^2_1).
\end{align*}
Finally we have 
\begin{align*}
 \phi(\hat{\p}) - \phi(\p^*) &  \leq  \epsilon_{byz}^2- \zeta^2_{byz}\phi(\p^*) \\
 \Rightarrow \phi(\hat{\p}) & \leq \epsilon_{byz}^2 + (1 - \zeta^2_{byz})\phi(\p^*).
\end{align*}

\end{proof}

\begin{lemma}\label{lem:byzdelta}
Let $\zeta_{byz} \in (0,1),\epsilon_{byz}$ be any fixed parameter. And $\hat{p}_t$ satisfies $\phi_t(\hat{\p}_t) \leq \epsilon_{byz}^2+ (1 -\zeta_{byz}^2) \min_{\p}\phi_t(\p)$. Under the Assumption~\ref{asm:hess}(Hessian $L$-Lipschitz) and 
$\mathbf{\Delta}_t =\w_t-\w^*$ satisfies
\begin{align*}
\mathbf{\Delta}^T_{t+1}\bH_t\mathbf{\Delta}_{t+1}& \leq L\| \mathbf{\Delta}_{t+1}\|\|\mathbf{\Delta}_t\|^2 +  \frac{\zeta^2_{byz}}{1-\zeta^2_{byz}}\mathbf{\Delta}_t^T\bH_t\mathbf{\Delta}_t + 2\epsilon_{byz}^2.
\end{align*}
\end{lemma}

\begin{proof}
 We choose $\zeta= \zeta_{byz}$ and $\epsilon=\epsilon_{byz}$ from the Lemma~\ref{lem:byzone} and follow  the proof of Lemma~\ref{lem:delta} to obtain the desired bound.
\end{proof}

\textbf{Proof of Theorem~\ref{thm:byzsmooth} }
\begin{proof}
 We get the desired bound by developing from the result of the Lemma~\ref{lem:byzdelta} and following the proof of Theorem~\ref{thm:smooth}
\end{proof}

 
\section{Appendix C:Analysis of Section~\ref{sec:compress}}

First we prove the following lemma that will be useful in our subsequent calculations. Consider that $\cQ(\hat{\p})= \frac{1}{|B|}\sum_{i \in B}\cQ(\hat{\p}_i)$. And also we use the following notation
$ \zeta_B= \nu(\frac{\eta}{\sqrt{|B|}}+ \frac{\eta^2}{1-\eta})$,  $\nu= \frac{\sigma_{max}(\bA^{\top}\bA)}{\sigma_{max}(\bA^{\top}\bA)+n\lambda} \leq 1$.

\begin{lemma}\label{lem: useful}
If $\cQ(\hat{\p}_i)$ is the local update direction and $\p^*$ is the optimal solution to the quadratic function $\phi$ then 
\begin{align*}
\left\| \bH^{\frac{1}{2}} (\cQ(\hat{\p}_i)-\p^* )\right\|^2 \leq  1+ \kappa (1-\rho) )\epsilon^2 + (\zeta^2_B + \kappa (1-\rho)((1+\zeta^2_1))\left\|\bH^{\frac{1}{2}} \p^* \right\|^2,
\end{align*} 
 where $\bH$ is the exact Hessian and 
  \begin{align*}
 \epsilon_1 &= \sqrt{(1+ \kappa (1-\rho) )}\epsilon, \\
 \zeta^2_{comp,B}&= (\zeta^2_B + \kappa (1-\rho)((1+\zeta^2_1)).
 \end{align*}
 $\epsilon$ is defined in equation~\eqref{eps} and 
 \end{lemma}
 
 \begin{proof}
 \begin{align}
 \left\| \bH^{\frac{1}{2}} (\cQ(\hat{\p})-\p^* )\right\|^2 & = \left\| \bH^{\frac{1}{2}} (\cQ(\hat{\p}) -\hat{\p} +\hat{\p}-\p^* )\right\|^2 \nonumber \\
& \leq 2 \left( \underbrace{ \left\| \bH^{\frac{1}{2}} (\cQ(\hat{\p}) -\hat{\p}) \right\|^2}_{Term 1}  + \underbrace{ \left\|\bH^{\frac{1}{2}}( \hat{\p}-\p^* )\right\|^2}_{Term 2} \right). \label{u1}
 \end{align}
Following the proof of Lemma~\ref{lem:onernd}  we get  
\begin{align}
 \left\| \bH^{\frac{1}{2}} (\hat{\p}_i -\p^*) \right\|^2 \leq \epsilon^2 +\zeta_1\left\|\bH^{\frac{1}{2}} \p^* \right\|^2,  \label{u2}
\end{align} 
 where $\epsilon$ is as  defined in \eqref{eps}.Now we consider the term
 \begin{align*}
 \left\| \bH^{\frac{1}{2}} (\cQ(\hat{\p}_i) -\hat{\p}_i \right\|^2 & \leq \sigma_{max}(\bH)(1-\rho)  \left\| \hat{\p}_i \right\|^2 \\
 & \leq \sigma_{max}(\bH)(1-\rho) \left(  \left\| \hat{\p}_i -\p^* \right\|^2 +  \left\| \p^* \right\|^2 \right) \\
 & \leq \frac{\sigma_{max}}{\sigma_{min}}(1-\rho) \left(  \left\| \bH^{\frac{1}{2}} (\hat{\p}_i -\p^*) \right\|^2 +  \left\|\bH^{\frac{1}{2}} \p^* \right\|^2 \right)\\
&= \kappa(1-\rho) \left(  \left\| \bH^{\frac{1}{2}} (\hat{\p}_i -\p^*) \right\|^2 +  \left\|\bH^{\frac{1}{2}} \p^* \right\|^2 \right)\\
& \leq  \kappa(1-\rho) \left(  \epsilon^2 +  (1+\zeta_1^2)\left\|\bH^{\frac{1}{2}} \p^* \right\|^2 \right) \quad \text{Using \eqref{u2}}.
 \end{align*}
 Now we use the above calculation and  bound  Term1 
\begin{align}
\left\| \bH^{\frac{1}{2}} (\cQ(\hat{\p}) -\hat{\p}) \right\|^2 & \leq \frac{1}{|B|}\sum_{i \in B} \left\| \bH^{\frac{1}{2}} (\cQ(\hat{\p}_i) -\hat{\p}_i \right\|^2 \nonumber \\
& \leq  \kappa (1-\rho) \left(  \epsilon^2 +  (1+\zeta^2_1)\left\|\bH^{\frac{1}{2}} \p^* \right\|^2 \right).  \label{u3}
\end{align} 
 We can bound the Term2  directly using the proof of  Lemma~\ref{lem:onernd}
 \begin{align}
 \left\|\bH^{\frac{1}{2}}( \hat{\p}-\p^* )\right\|^2 \leq \epsilon^2 + \zeta^2_B\left\|\bH^{\frac{1}{2}} \p^* \right\|^2. \label{u4}
 \end{align}
Now we use \eqref{u3} and \eqref{u4} and plug them in \eqref{u1} 
\begin{align*}
\left\| \bH^{\frac{1}{2}} (\cQ(\hat{\p})-\p^* )\right\|^2 & \leq (1+ \kappa (1-\rho) )\epsilon^2 + (\zeta^2_B + \kappa (1-\rho)((1+\zeta^2_1))\left\|\bH^{\frac{1}{2}} \p^* \right\|^2.
\end{align*}
 Now we define 
 \begin{align*}
 \epsilon_1 &= \sqrt{(1+ \kappa (1-\rho) )}\epsilon \\
 \zeta^2_{comp,B}&= (\zeta^2_B + \kappa (1-\rho)((1+\zeta^2_1)).
 \end{align*}
\end{proof}  
Now we have the robust update in iteration $t$ to be $\cQ(\hat{\p}) = \frac{1}{|\cU_t|} \sum_{i \in \cU_t}\cQ(\hat{p}_{i,t})$. 
\begin{lemma}\label{lem:combyzone}
Let $\{\bS_i\}_{i=1}^m \in \mathbb{R}^{n \times s}$  be sketching matrices based on  Lemma~\ref{lem:Hessketch}. Let $\phi_t$ be defined in \eqref{quad} and $\cQ(\hat{\p}_t)$ be the update with $\cQ$ being $\rho$-approximate compressor. It holds that 
\begin{align*}
\min_{\p}\phi_t(\p) \leq \phi_t(\cQ(\hat{\p}_t))& \leq \epsilon_{comp,byz}^2 + (1 - \zeta^2_{comp,byz})\phi_t(\p^*),
\end{align*}
where $\epsilon_{comp,byz}$ and $\zeta^2_{comp,byz}$ is as defined in ~\eqref{ceps} and ~\eqref{calpha} respectively.
\end{lemma}
\begin{proof}
In the following analysis we omit the subscript '$t$'. From the definition of the quadratic function \eqref{quad} we know that    
\begin{align*}
    \phi(\cQ(\hat{\p})) - \phi(\p^*) & = \frac{1}{2}\|\bH^{\frac12}(\cQ(\hat{\p})-\p^*)\|^2.
\end{align*}
Now we consider 
\vspace{-10pt}
\begin{align*}
  \frac{1}{2} \|\bH^{\frac12}(\cQ(\hat{\p})-\p^*)\|^2 & = \frac{1}{2}\|\bH^{\frac12}(\frac{1}{|\cU|}\sum_{i \in \cU}\cQ(\hat{\p}_i)-\p^*)\|^2  \\
   &=\frac{1}{2}\|\bH^{\frac12}\frac{1}{|\cU|}(\sum_{i \in \cM}(\cQ(\hat{\p}_i)-\p^*) -\sum_{i \in (\cM\cap\cT)}(\cQ(\hat{\p}_i)-\p^*)+\sum_{i \in (\cU\cap\cB)}(\cQ(\hat{\p}_i)-\p^*) )\|^2 \\
   &\leq \underbrace{ \|\bH^{\frac12}\frac{1}{|\cU|}(\sum_{i \in \cM}(\cQ(\hat{\p}_i)-\p^*)\|^2}_{Term 1} + \underbrace{ 2\|\bH^{\frac12}\frac{1}{|\cU|}\sum_{i \in (\cM\cap\cT)}(\cQ(\hat{\p}_i)-\p^*)\|^2 }_{Term 2} \\ 
   &+ \underbrace{2\|\bH^{\frac12}\frac{1}{|\cU|}\sum_{i \in (\cU\cap\cB)}(\cQ(\hat{\p}_i)-\p^*) )\|^2}_{Term 3}.
\end{align*}
Now we bound each term separately and use the Lemma~\ref{lem: useful}
\begin{align*}
    Term 1 &=  \|\bH^{\frac12}\frac{1}{|\cU|}(\sum_{i \in \cM}(\cQ(\hat{\p}_i)-\p^*)\|^2 \\& = (\frac{1-\alpha}{1-\beta})^2 \|\bH^{\frac12}\frac{1}{|\cM|}(\sum_{i \in \cM}(\cQ(\hat{\p}_i)-\p^*)\|^2 \\
   & \leq (\frac{1-\alpha}{1-\beta})^2 [\epsilon_1^2+ \zeta^2_{comp,\cM}\|\bH^{\frac{1}{2}}\p^*\|^2],
\end{align*}
where $\zeta^2_{comp,\cM} = (\zeta^2_{\cM} + \kappa (1-\rho)((1+\zeta^2_1)$. Similarly  the Term 2 can be bonded as it is a bound on good machines
\begin{align*}
    Term 2 & = 2\|\bH^{\frac12}\frac{1}{|\cU|}\sum_{i \in (\cM\cap\cT)}(\cQ(\hat{\p}_i)-\p^*)\|^2 \\
    & = 2 (\frac{1-\alpha}{1-\beta})^2 \|\bH^{\frac12}\frac{1}{|\cM\cap\cT|}\sum_{i \in (\cM\cap\cT)}(\cQ(\hat{\p}_i)-\p^*)\|^2\\
    & \leq 2(\frac{1-\alpha}{1-\beta})^2 [\epsilon_1^2 + \zeta^2_{comp,\cM\cap\cT}\|\bH^{\frac{1}{2}}\p^*\|^2].
\end{align*}
For the Term 3 we know that $\beta > \alpha $ so all the untrimmed worker norm is bounded by a good machine as at least one good machine gets trimmed.
\begin{align*}
    Term 3 & = 2\|\bH^{\frac12}\frac{1}{|\cU|}\sum_{i \in (\cU\cap\cB)}(\cQ(\hat{\p}_i)-\p^*) )\|^2 \\
    & \leq 2\sigma_{max}(\bH)(\frac{|\cU\cap\cB|}{|\cU|})^2 \|\frac{1}{|\cU\cap\cB|}\sum_{i \in (\cU\cap\cB)}(\cQ(\hat{\p}_i)-\p^*) )\|^2 \\
    & \leq 2\sigma_{max}(\bH)(\frac{|\cU\cap\cB|}{|\cU|})^2 \frac{1}{|\cU\cap\cB|}\sum_{i \in (\cU\cap\cB)}\|(\cQ(\hat{\p}_i)-\p^*) )\|^2\\
   & \leq 4\sigma_{max}(\bH)(\frac{|\cU\cap\cB|}{|\cU|})^2 \frac{1}{|\cU\cap\cB|}\sum_{i \in (\cU\cap\cB)}(\|\cQ(\hat{\p}_i)\|^2+ \|\p^*\|^2) \\
   & \leq 4\sigma_{max}(\bH)(\frac{|\cU\cap\cB|}{|\cU|})^2 \max_{i \in \cM}(\|\cQ(\hat{\p}_i)\|^2+ \|\p^*\|^2)\\
   & \leq 4\sigma_{max}(\bH)(\frac{|\cU\cap\cB|}{|\cU|})^2 \max_{i \in \cM}(\|\cQ(\hat{\p}_i)- \p^*\|^2+ 2\|\p^*\|^2)\\
   & \leq 4\kappa (\frac{|\cU\cap\cB|}{|\cU|})^2 \max_{i \in \cM}(\|\bH^{\frac{1}{2}}(\cQ(\hat{\p}_i)- \p^*)\|^2+ 2\|\bH^{\frac{1}{2}}\p^*\|^2)\\
   &\leq   4\kappa (\frac{|\cU\cap\cB|}{|\cU|})^2(\epsilon_1^2 + (2+\zeta^2_{1})\|\bH^{\frac{1}{2}}\p^*\|^2)\\
   &\leq   4\kappa (\frac{\alpha}{1-\beta})^2(\epsilon_1^2 + (2+\zeta^2_{1})\|\bH^{\frac{1}{2}}\p^*\|^2).
\end{align*}
Combining all the bounds on Term1 , Term2 and Term3 we have 
\begin{align*}
  \frac{1}{2} \|\bH^{\frac12}(\hat{\p}-\p^*)\|^2 & \leq    \epsilon_{byz}^2+ \zeta^2_{byz}\|\bH^{\frac{1}{2}}\p^*\|^2,
\end{align*}
where 
\begin{align*}
    \epsilon_{comp,byz}^2 &= \left(3 \left(\frac{1-\alpha}{1-\beta} \right)^2 +4\kappa \left(\frac{\alpha}{1-\beta}\right)^2 \right)\epsilon_1^2 \\
    \zeta^2_{comp,byz} & =2\left(\frac{1-\alpha}{1-\beta}\right)^2\zeta^2_{comp,\cM\cap\cT} + \left(\frac{1-\alpha}{1-\beta}\right)^2\zeta^2_{comp,\cM} + 4\kappa  \left(\frac{\alpha}{1-\beta}\right)^2 (2+\zeta^2_{comp,1}).
\end{align*}
Finally we have 
\begin{align*}
 \phi(\hat{\p}) - \phi(\p^*) &  \leq  \epsilon_{comp,byz}^2- \zeta^2_{comp,byz}\phi(\p^*) \\
 \Rightarrow \phi(\hat{\p}) & \leq \epsilon_{comp,byz}^2 + (1 - \zeta^2_{comp,byz})\phi(\p^*).
\end{align*}
\end{proof}

\begin{lemma}\label{lem:compbyzdelta}
Let $\zeta_{comp,byz} \in (0,1),\epsilon_{comp,byz}$ be any fixed parameter. And $\cQ(\hat{p}_t)$ satisfies $\phi_t(\cQ(\hat{p}_t)) \leq \epsilon_{byz}^2+ (1 -\zeta_{byz}^2) \min_{\p}\phi_t(\p)$. Under the Assumption~\ref{asm:hess}(Hessian $L$-Lipschitz) and 
$\mathbf{\Delta}_t =\w_t-\w^*$ satisfies
\begin{align*}
\mathbf{\Delta}^T_{t+1}\bH_t\mathbf{\Delta}_{t+1}& \leq L\| \mathbf{\Delta}_{t+1}\|\|\mathbf{\Delta}_t\|^2 +  \frac{\zeta^2_{comp,byz}}{1-\zeta^2_{comp,byz}}\mathbf{\Delta}_t^T\bH_t\mathbf{\Delta}_t + 2\epsilon_{comp,byz}^2.
\end{align*}
\end{lemma}

\begin{proof}
 We choose $\zeta= \zeta_{comp,byz}$ and $\epsilon=\epsilon_{comp,byz}$ from the Lemma~\ref{lem:combyzone} and follow  the proof of Lemma~\ref{lem:delta} to obtain the desired bound.
\end{proof}

\textbf{Proof of Theorem~\ref{thm:compbyzsmooth} }
\begin{proof}
 We get the desired bound by developing from the result of the Lemma~\ref{lem:compbyzdelta} and following the proof of Theorem~\ref{thm:smooth}
\end{proof}

\section{Additional Experiment}

In addition to the experimental results in Section~\ref{sec:exp}, we provide some more experiments  supporting the robustness of the \textsf{COMRADE} in two different types of attacks : 1. `Gaussian attack': where the Byzantine workers add Gaussian Noise $(\mathcal{N}(\mu,\sigma^2))$ to the update and 2. `random label attack': where the Byzantine worker machines learns based on random labels instead of proper labels.  

\begin{figure}[h!]%
    \centering
   \subfloat[w5a `Gauss']{{\includegraphics[height = 3cm,width=3.5cm]{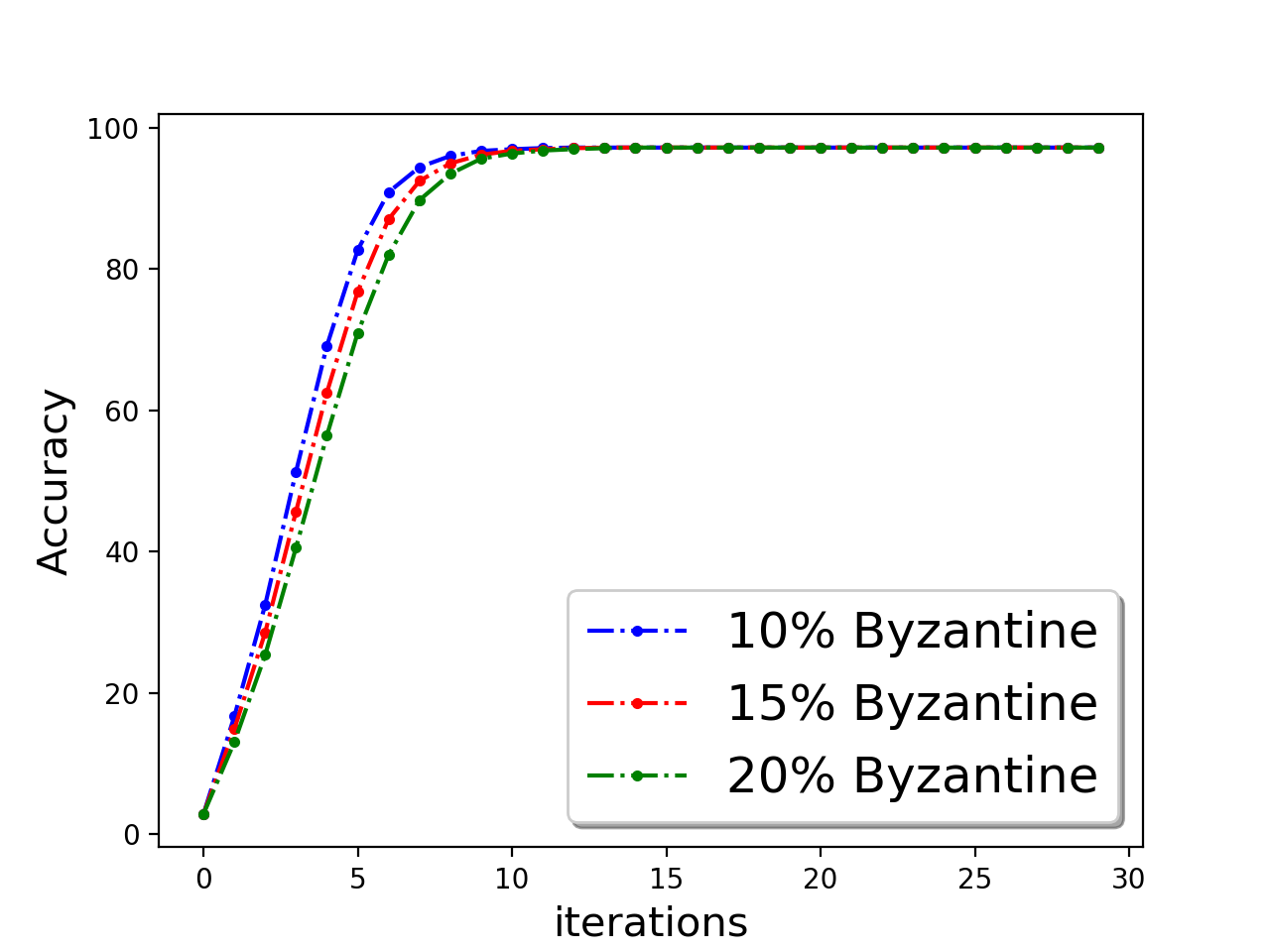} }}%
    \subfloat[a9a `Gauss']{{\includegraphics[height = 3cm,width=3.5cm]{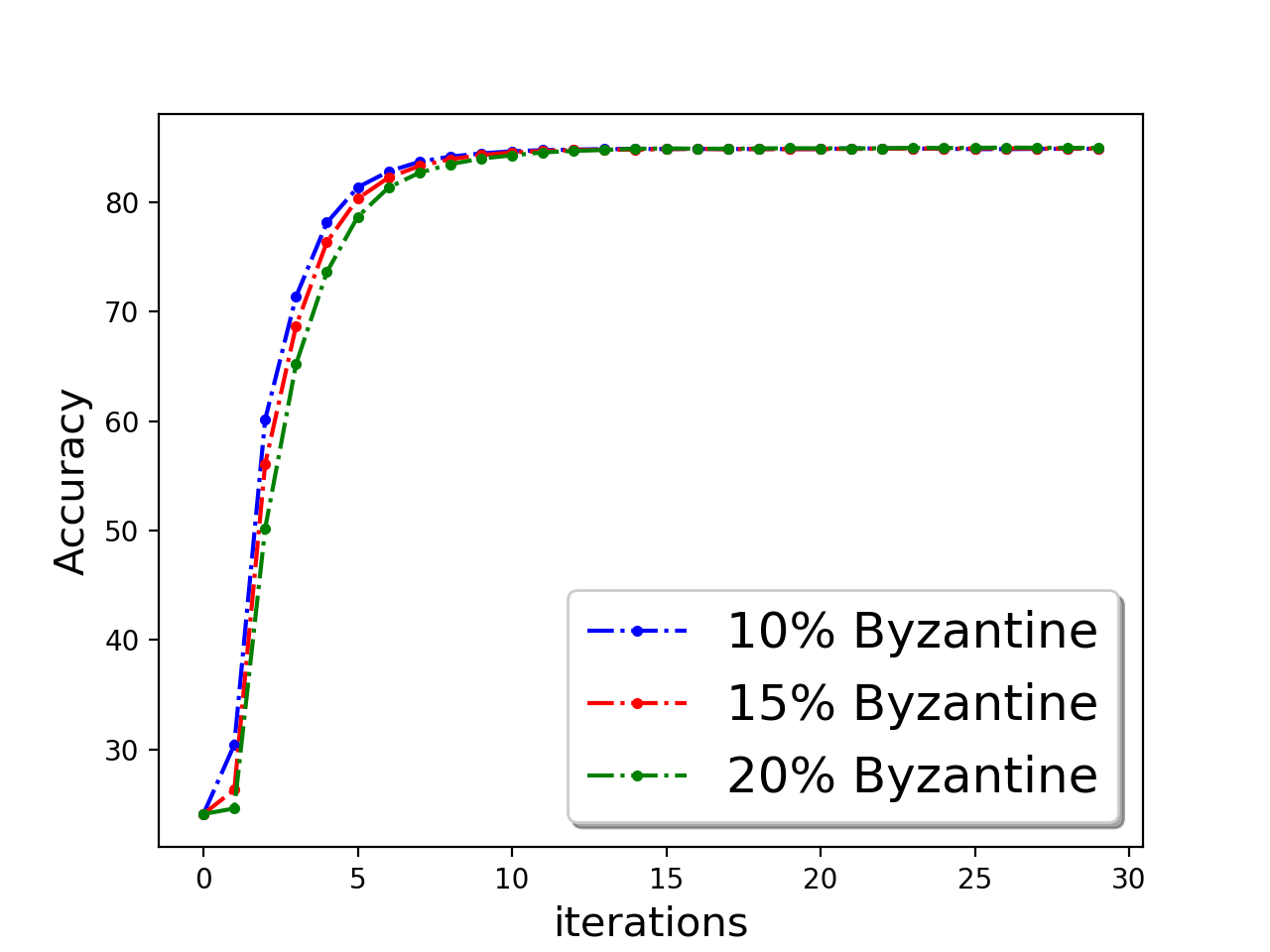} }}%
  \subfloat[w5a `random']{{\includegraphics[height = 3cm,width=3.5cm]{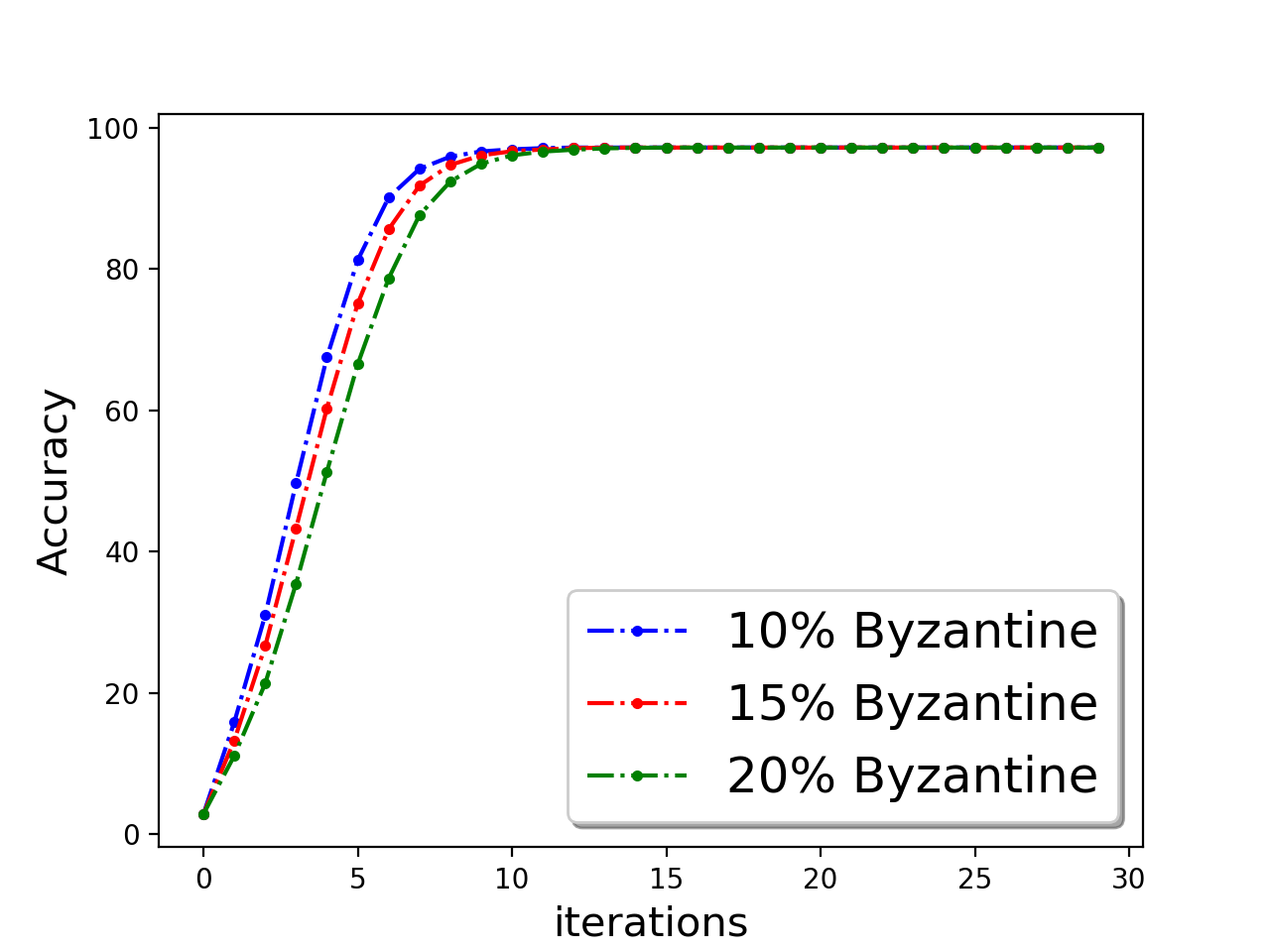} }}%
    \subfloat[a9a `random']{{\includegraphics[height = 3cm,width=3.5cm]{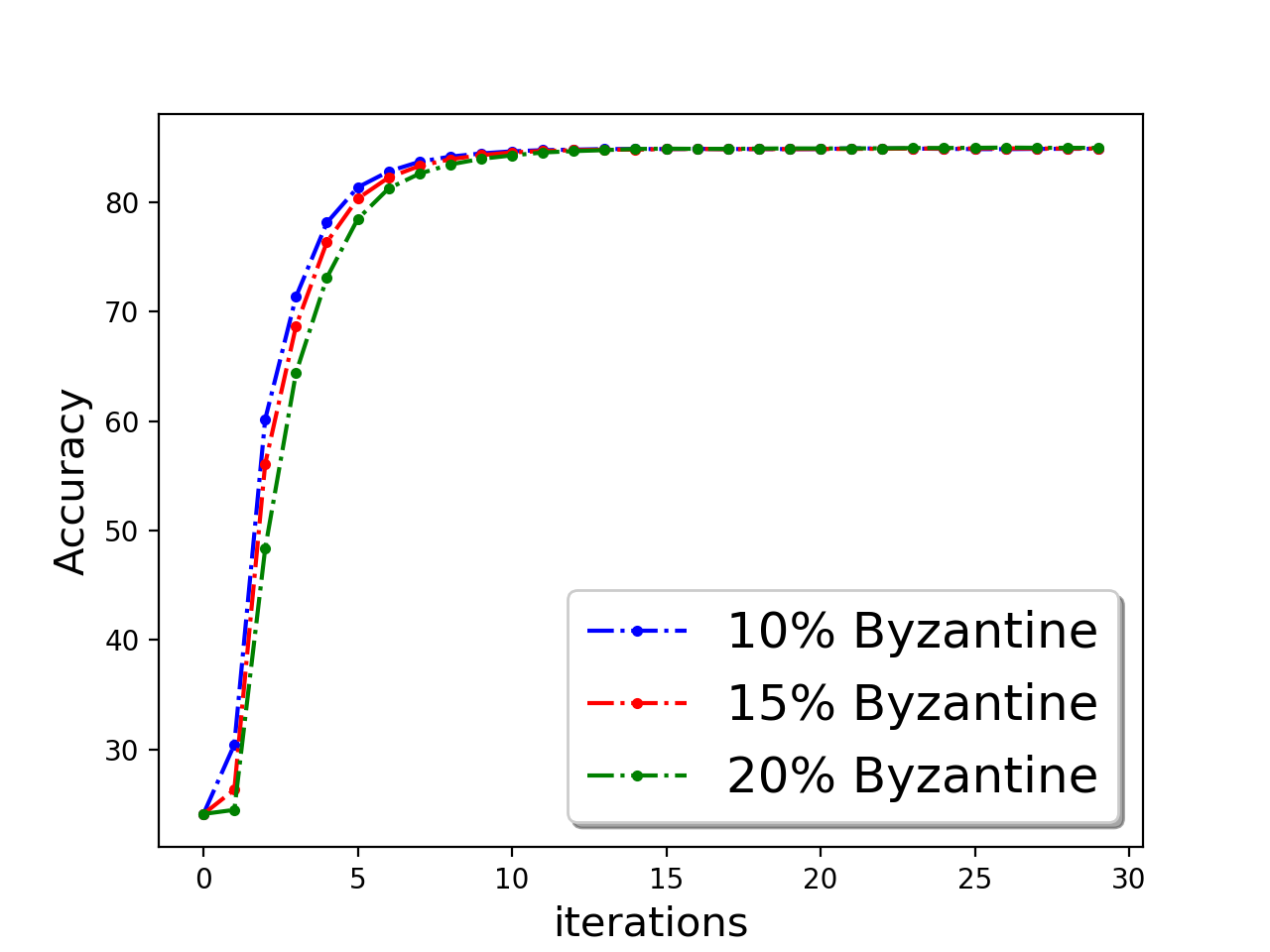} }}%
\vspace{-10pt}
    \subfloat[w5a `Gauss']{{\includegraphics[height = 3cm,width=3.5cm]{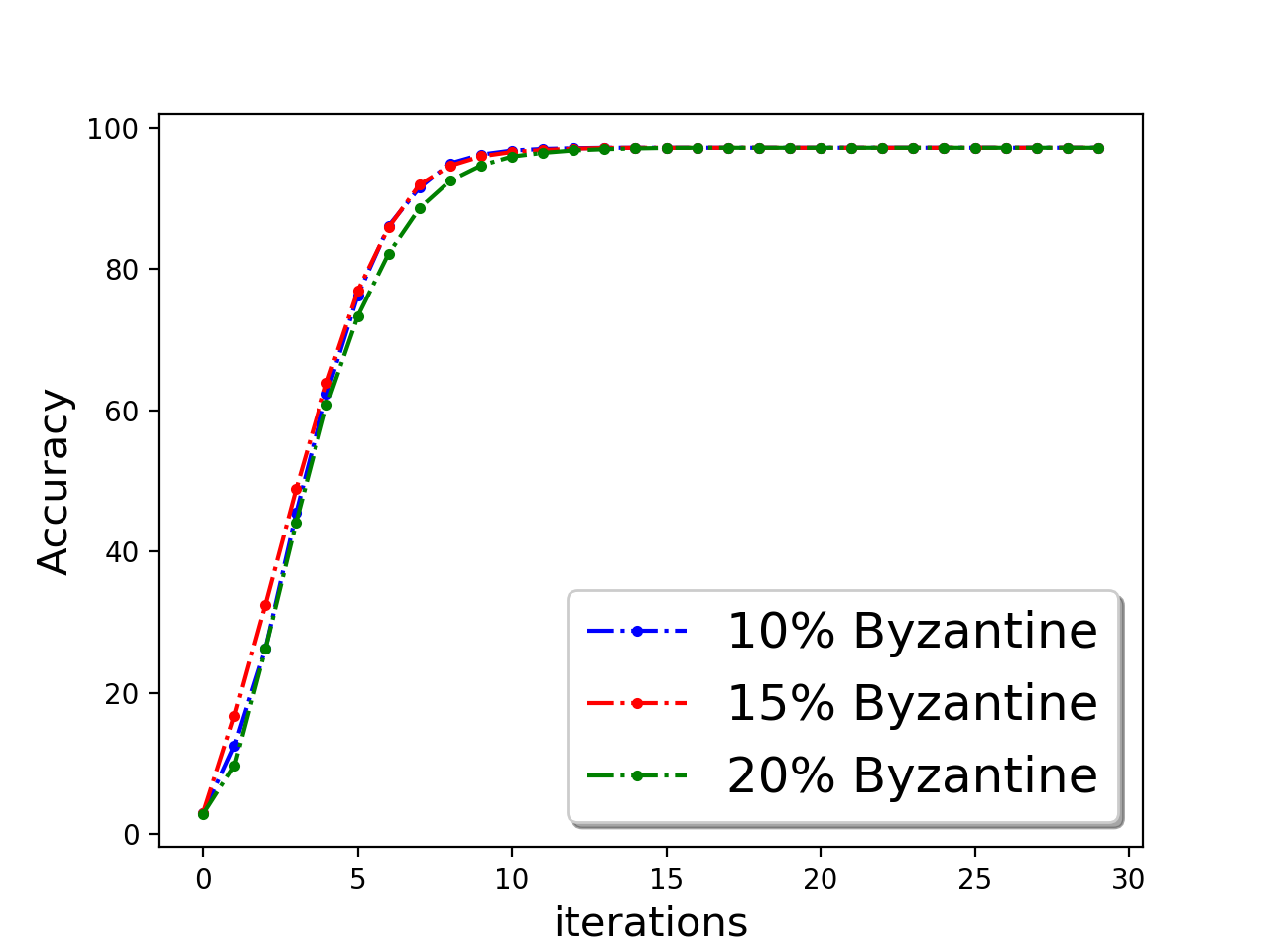} }}%
    \subfloat[a9a `Gauss']{{\includegraphics[height = 3cm,width=3.5cm]{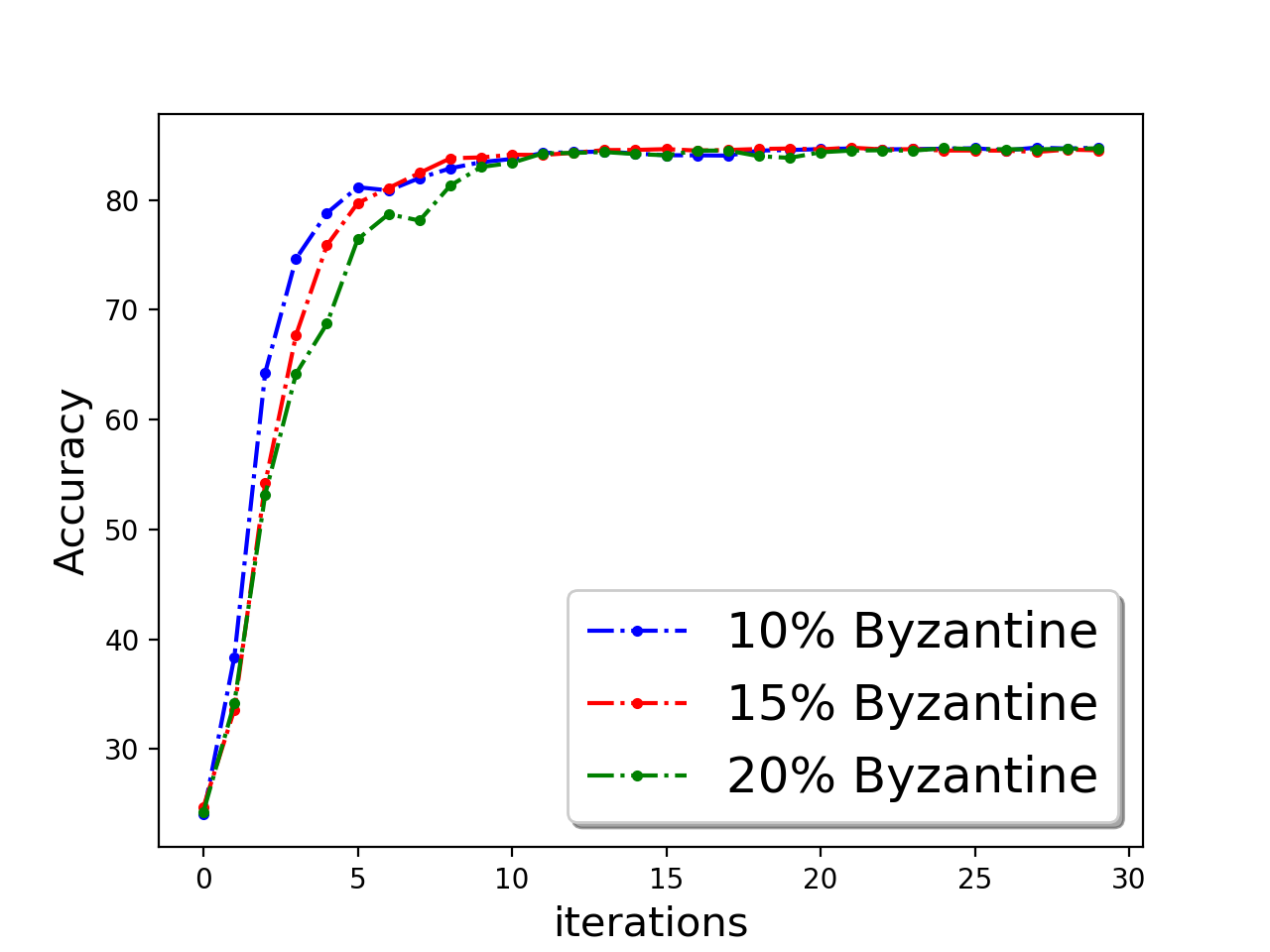} }}%
    \subfloat[w5a `random']{{\includegraphics[height = 3cm,width=3.5cm]{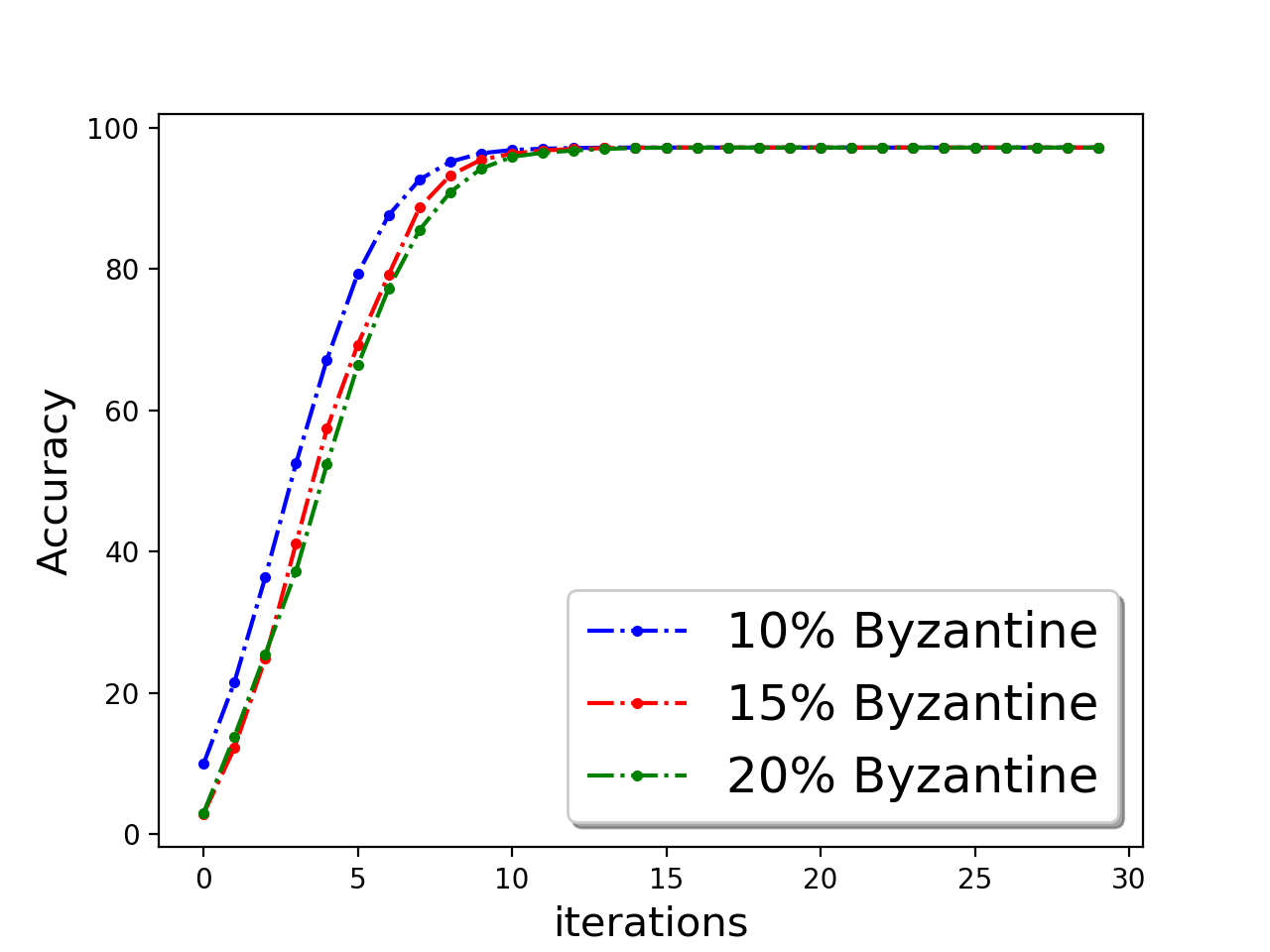} }}%
    \subfloat[a9a `random']{{\includegraphics[height = 3cm,width=3.5cm]{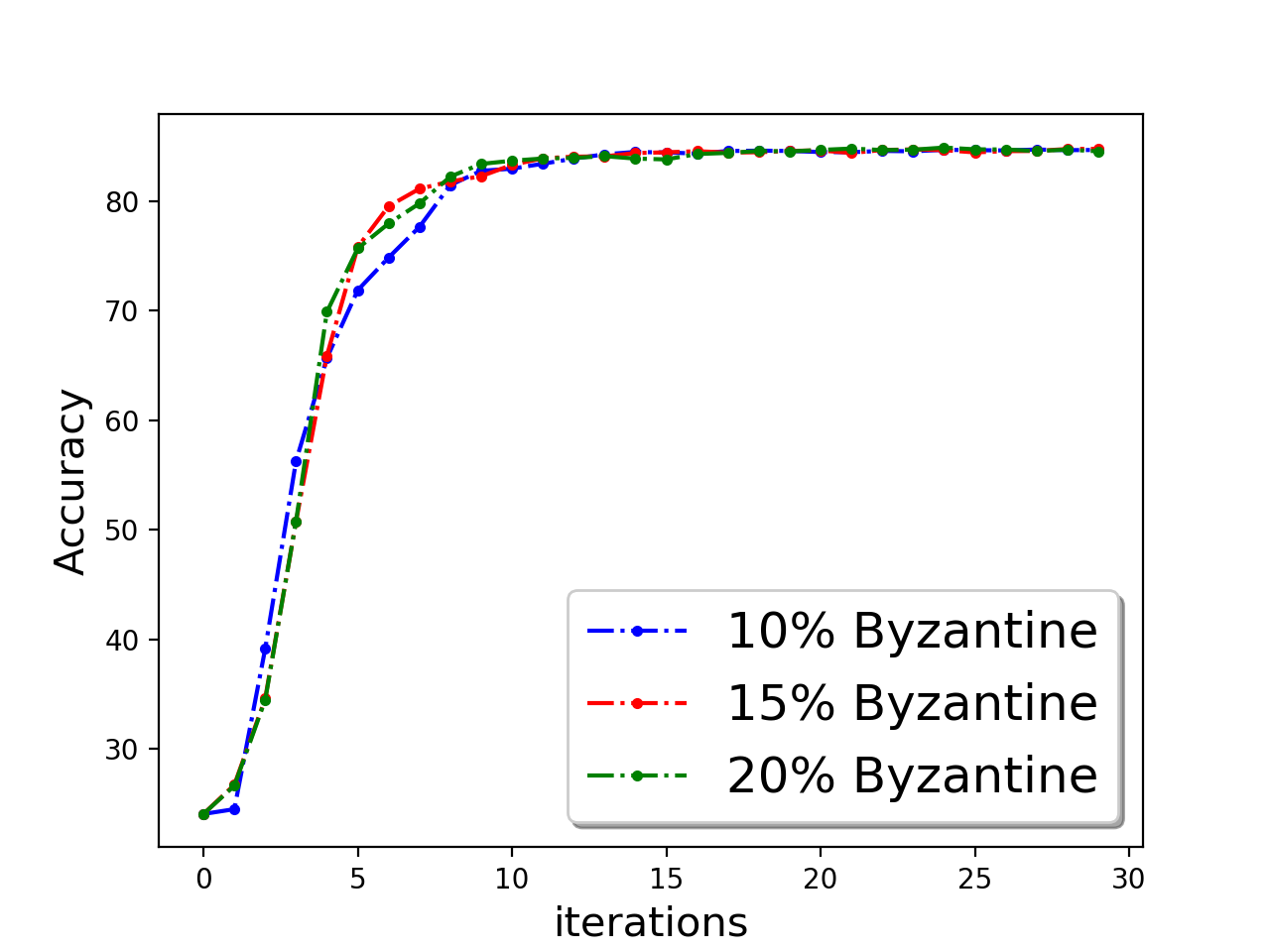} }}%
    \caption{(First row) Accuracy of \textsf{COMRADE}  with $10\%,15\%,20\%$ Byzantine workers with  `Gaussian ' attack for (a).  w5a (b). a9a and `random label' attack for (c). w5a (d).a9a.  (Second row) Accuracy of \textsf{COMRADE} with $\rho$-approximate compressor (Section~\ref{sec:compress}) with $10\%,15\%,20\%$ Byzantine  workers  with  `Gaussian ' attack for (a).  w5a (b). a9a and `random label' attack for (c). w5a (d).a9a.}%
    \label{fig:byz1}%
\end{figure}
\end{document}